\documentclass[nohyperref]{article}

\usepackage{microtype}
\usepackage{graphicx}
\usepackage{subfigure}
\usepackage{booktabs, multirow} 
\usepackage{adjustbox}
\usepackage{hyperref}
\usepackage{thmtools, thm-restate}

\usepackage[accepted]{icml2023}

\usepackage{amsmath}
\usepackage{amssymb}
\usepackage{mathtools}
\usepackage{amsthm}
\usepackage{comment}
\usepackage[textsize=tiny]{todonotes}
\usepackage[capitalize,noabbrev]{cleveref}

\usepackage{bm}

\theoremstyle{plain}
\newtheorem{theorem}{Theorem}[section]

\newtheorem{lemma}[theorem]{Lemma}
\newtheorem{corollary}[theorem]{Corollary}
\theoremstyle{definition}
\newtheorem{definition}[theorem]{Definition}

\theoremstyle{remark}

\usepackage[textsize=tiny]{todonotes}

\icmltitlerunning{Masked Bayesian Neural Networks}
\begin{document}

\twocolumn[
\icmltitle{Masked Bayesian Neural Networks : \\
Theoretical Guarantee and its Posterior Inference}



\icmlsetsymbol{equal}{*}

\begin{icmlauthorlist}
\icmlauthor{Insung Kong}{snu}
\icmlauthor{Dongyoon Yang}{snu}
\icmlauthor{Jongjin Lee}{snu}
\icmlauthor{Ilsang Ohn}{Inha}
\icmlauthor{Gyuseung Baek}{obzen}
\icmlauthor{Yongdai Kim}{snu}
\end{icmlauthorlist}

\icmlaffiliation{snu}{Department of Statistics, Seoul National University}
\icmlaffiliation{Inha}{Department of Statistics, Inha University}
\icmlaffiliation{obzen}{Obzen}

\icmlcorrespondingauthor{Yongdai Kim}{ydkim0903@gmail.com}

\icmlkeywords{Machine Learning, ICML}

\vskip 0.3in
]



\printAffiliationsAndNotice{}  

\begin{abstract}

Bayesian approaches for learning deep neural networks (BNN)
have been received much attention and successfully applied to various applications. Particularly, BNNs have the merit of having better generalization ability as well as better uncertainty quantification. For the success of BNN,
search an appropriate architecture of the neural networks is an important task, and
various algorithms to find good sparse neural networks have been proposed.
In this paper, we propose a new node-sparse BNN model which has good theoretical properties and is computationally feasible. 
We prove that the posterior concentration rate to the true model is near minimax optimal and adaptive to the smoothness of the true model. 
In particular the adaptiveness is the first of its kind for node-sparse BNNs.
In addition, we develop a novel MCMC algorithm which makes the Bayesian inference of the node-sparse BNN model feasible in practice. 
\end{abstract}

\section{Introduction}
\label{sec1}

Bayesian approaches for learning deep neural networks (DNN), which is called Bayesian Neural Networks (BNN) \cite{mackay1992practical, neal2012bayesian},
have been received much attention and successfully applied to various applications. 
Particularly, BNNs have the merit of having better generalization ability as well as better uncertainty quantification \cite{wilson2020bayesian, izmailov2021bayesian}. Applications of BNNs range from recommender systems \cite{wang2015collaborative} to topic modeling \cite{gan2015scalable}, medical diagnosis \cite{filos2019systematic}, and astrophysics \cite{cranmer2021bayesian}, to name just a few. 

Various BNN models to 
have desirable theoretical properties 
have been proposed. 
In particular, edge-sparse BNNs are the main focus of research of BNNs.
Fast posterior concentrate rates of edge-sparse BNNs have been studied by
\cite{polson2018posterior, cherief2020convergence, bai2020efficient, lee2022asymptotic}. 
Particularly, \citet{polson2018posterior} uses Spike-and-Slab prior on each edge and prove the near minimax concentrate rate of the posterior, but posterior computation is difficult. 
To ease computation, \citet{molchanov2017variational} uses variational dropout \cite{kingma2015variational} to search good edge-sparse BNNs 
and \citet{deng2019adaptive} and \citet{wang2021bayesian} develop edge-sparse learning algorithms via adaptive empirical Bayesian methods. 
\citet{nalisnick2019dropout} shows that multiplicative noise induces structured shrinkage priors on a network’s weights.
However, theoretical justifications of these algorithms lack.

Node-sparse BNNs are useful alternatives to edge-sparse BNNs because their
inferential costs are lighter than edge-sparse BNNs.
Using scale-mixture priors, \citet{louizos2017bayesian} and \citet{ghosh2019model} develop
node-sparse BNNs based on the VI, but theoretical justifications of these algorithms are not available.
\citet{jantre2021layer} derives the posterior concentration
rate of  a node-sparse VI approach using Spike-and-Slab prior.
However, their result does not meet the minimax optimal rate and is not adaptive to the smoothness
of the true model.

In this paper, we propose a new node-sparse BNN model called the masked BNN (mBNN) which has good theoretical properties and is computationally feasible. 
We prove that the posterior concentration rate to the true model is near minimax optimal and adaptive to the smoothness of the true model. 
The adaptiveness, which is the first of its kind for node-sparse BNNs, means that the mBNN selects optimal sparse architectures without knowing the complexity of the true model.
Moreover, we develop a novel MCMC algorithm which makes the Bayesian inference of the node-sparse BNN model possible. 

Our proposed node-sparse BNN can be also used for compressing complex DNNs.
Mots recent DNNs are based on complex network architectures consisting of multiple 
non-linear hidden layers, which results in expensive computation costs and requirements of excessive storage capacities when they are deployed to application systems. 
Various methodologies to alleviate memory and computation costs
have been suggested.
A popular approach is to prune iteratively unnecessary edges 
\cite{han2015learning, frankle2018lottery} or unnecessary nodes \cite{wen2016learning, he2017channel, wang2019eigendamage, chin2020towards}. 
However, such algorithms do not provide 
proper uncertainty quantification. By analyzing image data with CNNs,
we illustrate that our node-sparse BNN is good at compressing DNNs  without hampering
the ability of uncertainty quantification.

Our contributions are summarized as follows:
\begin{itemize}
	\item We develop a node-sparse prior for DNNs such that the posterior concentration rate to the true model is near minimax optimal adaptively to the smoothness of the true model and Bayesian inference with a specially designed MCMC algorithm is possible. 
 	

	\item We implement an efficient MCMC algorithm for searching good node-sparse BNNs. In particular, we develop a local informed proposal distribution in the Metropolis-Hastings (MH) algorithm to search good node-sparse architectures efficiently.

	\item By numerical experiments, we illustrate that the mBNN outperforms other Bayesian approaches 
    including nonsparse BNN and VI algorithms in terms of generalization and uncertainty quantification.
	 
\end{itemize}

\section{Preliminaries}
\label{sec2}
\subsection{Notation}
Let $\mathbb{R}$ and $\mathbb{N}$ be the sets of real numbers and natural numbers, respectively. For an integer $n \in \mathbb{N}$, we denote $[n] := \{1,\dots,n\}$. A capital letter denotes a random variable or matrix interchangeably
whenever its meaning is clear, and a vector is denoted by a bold letter, e.g. $\bm{x} := (x_1 , \dots, x_d)^{\top}$. 
For a $d$-dimensional vector $\bm{x} \in \mathbb{R}^d$, we denote $|\bm{x}|_p := (\sum_{j=1}^{d}|x_j|^p)^{1/p}$ for $1 \leq p < \infty$, $|\bm{x}|_{0} := \sum_{j=1}^{d }\mathbb{I}(x_j \ne 0)$ and $|\bm{x}|_{\infty} := \max_{j \in [d]}|x_j|$. 
For a real-valued function  $f : \mathcal{X} \to \mathbb{R}$  and $1 \leq p < \infty$, we denote $||f||_{p,n} := (\sum_{i=1}^{n} f(\bm{x}_i)^p/n)^{1/p}$ and $||f||_{p,\mathrm{P}_{\bm{X}}} := \left(\int_{\bm{X} \in \mathcal{X}} f(\bm{X})^p d\mathrm{P}_{\bm{X}} \right)^{1/p}$ where
$\mathrm{P}_{\bm{X}}$ is a probability measure defined on input space $\mathcal{X}$.
Moreover, we define $||f||_{\infty} = \operatorname{sup}_{\bm{x} \in \mathcal{X}}|f(\bm{x})|.$ 
For $b_1 \leq b_2$, we define $f_{[b_1 , b_2 ]}(\cdot) := \min( \max(f(\cdot),b_1), b_2)$ which is a truncated version of $f$ on $b_1$ and $b_2$. 
We denote $\circ$ and $\odot$ as the composition of functions and element-wise product
of vectors or matrices, respectively.
For a probability vector $\bm{q} \in \mathbb{R}^r$, 
we denote $\operatorname{Cat}(\bm{q})$ as the categorical distribution with the
probabilities of each category being $\bm{q}$ and
$\operatorname{Multi}_{\left( N, \bm{q} \right)}$
as the distribution of N many selected balls without replacement from the jar containing $r$ many balls whose selection probabilities are $\bm{q}$.
For technical simplicity, we assume $\mathcal{X} = [-1,1]^d$.

\subsection{Data generating process} \label{sec_generative}
We consider two supervised learning problems : regression and classification. 
In regression problems, the input vector $\bm{X}$ and the response variable $Y \in \mathbb{R}$ are generated from the model
\begin{align}
\begin{split} \label{reg}
\bm{X} \sim & \mathrm{P}_{\bm{X}}, \\
Y|\bm{X} \sim & N(f_0 (\boldsymbol{X}), \sigma_0^2), 
\end{split}
\end{align}
where $\mathrm{P}_{\bm{X}}$ is the probability measure defined on $\mathcal{X}.$ Here, $f_0 : \mathcal{X} \to \mathbb{R}$ and $\sigma_0^2 > 0$ are the unknown true regression function and unknown variance of the noise, respectively.

For $K$-class classification problems, the input vector $\bm{X}$ and the response variable $Y \in [K]$ are generated from the model
\begin{align}
\begin{split} \label{cla}
\bm{X} \sim & \mathrm{P}_{\bm{X}}, \\
Y|\bm{X} \sim & \operatorname{Cat}\left(\operatorname{softmax}(\bm{f}_0 (\boldsymbol{X}))\right),
\end{split}
\end{align}
where $\mathrm{P}_{\bm{X}}$ is the probability measure defined on $\mathcal{X}$ and  $\bm{f}_0 : \mathcal{X} \to \mathbb{R}^K$ is the logit of the unknown true conditional class probability function.

\section{Masked Bayesian Neural Network}
\label{sec3}

To construct a node-sparse BNN, we propose to use masking vectors which screen some nodes of the hidden layers. We first define the masked Deep Neural Network (mDNN) model, and then propose the mBNN on the top of the mDNN by specifying a prior appropriately.

\subsection{Deep Neural Network}

For $L \in \mathbb{N}$ and $\bm{p} = (p^{(0)}, p^{(1)}, ... , p^{(L)}, p^{(L+1)})^{\top} \in \mathbb{N}^{L+2}$, DNN with the $(L, \bm{p})$ architecture is a DNN model which has $L$ hidden layers and 
$p^{(l)}$ many nodes at the $l$-th hidden layer for $l\in [L].$
The input and output dimensions are $p^{(0)}$ and $p^{({L+1})},$ respectively.
The output of the DNN model can be written as
\begin{equation}\label{DNN}
f_{\bm{\theta}}^{\operatorname{DNN}}(\cdot) := A_{L+1} \circ \rho \circ A_{L} \dots \circ \rho \circ A_{1} (\cdot),
\end{equation}
where $A_l : \mathbb{R}^{p^{(l-1)}} \mapsto \mathbb{R}^{p^{(l)}}$ for $l \in [L+1]$ is an affine map defined as $A_l (\boldsymbol{x}) := W_l \bm{x} + \bm{b}_l$ with $W_l \in \mathbb{R}^{p^{(l)} \times p^{(l-1)}}$ and $\bm{b}_l \in \mathbb{R}^{p^{(l)}}$ 
and $\rho$ is the RELU activation function.
The DNN model is parameterized by $\bm{\theta}$ which is the concatenation of the weight matrices and bias vectors, that is
\begin{equation*}
\bm{\theta} := (\operatorname{vec}(W_1)^{\top}, b_1^{\top},\dots,\operatorname{vec}(W_{L+1})^{\top}, b_{L+1}^{\top})^{\top}.
\end{equation*}

\subsection{Masked Deep Neural Network} \label{sec_3.2}
For a given standard DNN, the corresponding masked DNN (mDNN) is constructed
by simply adding masking parameters to the standard DNN model.
For $l \in [L]$, the mDNN model screens
the nodes at the $l$-th hidden layer using the binary masking vector $\bm{m}^{(l)} \in \{0,1\}^{p^{(l)}}$. When        
$\left(\bm{m}^{(l)}\right)_j = 0,$ 
the $j$-th node at the $l$-th hidden layer becomes inactive.
The output of the mDNN model with the $(L, \bm{p})$ architecture can be written as
\begin{equation*}
f_{\bm{M}, \bm{\theta}}^{\operatorname{mDNN}}(\cdot) := A_{L+1} \circ \rho_{\bm{m}^{(L)}} \circ A_{L} \dots \circ \rho_{\bm{m}^{(1)}} \circ A_{1} (\cdot),
\end{equation*}
where $A_l$ for $l \in [L+1]$ is the affine map defined in (\ref{DNN})
and $\rho_{\bm{m}^{(l)}} : \mathbb{R}^{p^{(l)}} \to \mathbb{R}^{p^{(l)}}$ for $l \in [L]$ is the masked-RELU activation function defined as
\begin{equation*}
\rho_{\bm{m}^{(l)}} \left(\begin{array}{c}
x_{1} \\ \vdots \\ x_{p^{(l)}} \end{array}\right)
:=\bm{m}^{(l)} \odot \left(\begin{array}{c}
\max(x_1,0) \\ \vdots \\ \max(x_{p^{(l)}},0) \end{array}\right).
\end{equation*}
The model is parameterized by $\bm{M}$ and $\bm{\theta}$, 
where $\bm{M} \in \{0,1\}^{\sum_{l=1}^{L} p^{(l)}}$ is the concatenate of the all masking vectors
\begin{equation*}
    \bm{M} := \left({\bm{m}^{(1)}}^{\top} , \dots, {\bm{m}^{(L)}}^{\top}\right)^{\top}
\end{equation*}
and $\bm{\theta}$ is the concatenation of the weight matrices and the bias vectors in the standard DNN model.

Note that the mDNN is nothing but a standard DNN with the architecture
$(L,\tilde{\bm{p}})$
where $(\tilde{\bm{p}})_l=\left|\bm{m}^{(l)}\right|_{0}.$
That is, the mDNN is a reparameterization of the standard DNN using the masking vectors. 
This reparameterization, however, allows us to develop an efficient MCMC algorithm, in particular for searching good architectures (i.e. good masking vectors).

\subsection{Prior and posterior distribution} \label{section_prior}

We adopt a data-dependent prior $\Pi_n$ on the parameters $\bm{M}$ and $\bm{\theta}$ 
(as well as $\sigma^2$ for regression problems). We consider a data-dependent prior to ensure 
the optimal posterior concentration rate and adaptiveness. 
We assume that a priori $\bm{M}$ and $\bm{\theta}$ 
(as well as $\sigma^2$ for regression problems) are independent.
 
For $\bm{M},$ we use the following hierarchical prior. For each $\bm{m}^{(l)}$ with $l \in [L]$, let $s^{(l)} := |\bm{m}^{(l)}|_0$ be the sparsity of the masking vector $\bm{m}^{(l)}$. We put a prior mass on $s^{(l)}$ by    
\begin{align} \label{prior_s}
\Pi_n (s^{(l)}) {\propto} e^{-(\lambda \log n)^5 {s^{(l)}}^2} \qquad \text{ for } s^{(l)} \in [p^{(l)}],
\end{align}
where $\lambda>0$ is a hyper-parameter.
Note that the prior on $s^{(l)}$ regularizes the width of the network, and more strong regularization is enforced as more data are accumulated.
Given the sparsity level $s^{(l)}$, the masking vector $\bm{m}^{(l)}$ is sampled from the set $\{\bm{m}^{(l)} \in \{0,1\}^{p^{(l)}} : |\bm{m}^{(l)}|_0 =  s^{(l)}\}$ uniformly. In other words,
\begin{align} \label{prior_m}
\Pi_n (\bm{m}^{(l)} | s^{(l)}) \stackrel{iid}{\propto}& \ \frac{1}{{p^{(l)} \choose s^{(l)}}} \mathbb{I}(|\bm{m}^{(l)}|_0 = s^{(l)})
\end{align}
and the prior of $\bm{m}^{(l)}$ is the product of (\ref{prior_s}) and (\ref{prior_m}) 

For the prior of $\bm{\theta},$ we assume that 
\begin{align} \label{prior_theta}
\theta_i \stackrel{iid}{\sim} \mathfrak{p}(\theta_i)  \qquad i \in \left[T\right],
\end{align}
where $T := \sum_{l=0}^{L} (p^{(l)}+1)p^{(l+1)}$ is the length of the vector $\bm{\theta},$ and choose $\mathfrak{p}$ carefully to ensure desirable
theoretical properties.
For high-dimensional linear regression problems, \citet{castillo2012needles} and \citet{castillo2015bayesian} notice that using a heavy-tailed distribution for the prior of the regression coefficients is essential for theoretical optimality. Motivated by these observations, we consider a heavy-tailed distribution for $\mathfrak{p}.$
Let $\mathfrak{P}$ be the class of polynomial tail distributions on $\mathbb{R}$ defined as
\begin{align*} 
    \mathfrak{P} := \left\{ \mathfrak{p} : \lim_{x \to \infty} \frac{x^{-\log x}}{\mathfrak{p}(x)} \to 0 \text{ and } \lim_{x \to \infty} \frac{x^{-\log x}}{\mathfrak{p}(-x)} \to 0 \right\}.
\end{align*}
Examples of polynomial tail distributions are the Cauchy distribution and Student's t-distribution. On the other hand, the Gaussian and Laplace distributions do not belong to $\mathfrak{P}.$
We assume that $\mathfrak{p}$ belongs to $\mathfrak{P}.$
We will show in Section \ref{sec4} that any prior in $\mathfrak{P}$
yields the optimal posterior concentration rate. 

For the prior of $\sigma^2$ in regression problems,
a standard distribution such as the inverse-gamma distribution can be used.
Any distribution whose density is positive at the true $\sigma^2_0$
works for theoretical optimality.



\subsection{Comparison with MC-dropout}
The idea of masking nodes in DNN has been already used in various algorithms. 
Dropout \cite{srivastava2014dropout} and MC-dropout \cite{gal2016dropout}
are two representative examples, where they
randomly mask the nodes of DNN during the training phase. While Dropout abolishes the masking vectors and uses the scaled-down version of the trained weights
in the prediction phase, 
MC-dropout uses the trained weights obtained in the training phase multiplied by a random masking vectors. 
Since the masking vectors is treated as a random vectors following its posterior distribution at the prediction phase, our mBNN is similar to MC-dropout.
A key difference between the mBNN and MC-dropout, however, is that
the mBNN learns the distribution of the masking vectors from data 
via the posterior distribution but MC-dropout does not.
That is, the mBNN learns the architecture of DNN from data, which makes
the mBNN have good theoretical and empirical properties.

\section{Theoretical optimalities}
\label{sec4}
In this section, we derive the posterior concentration rates of the mBNNs for regression and classification problems, which are minimax optimal up to a logarithmic factor. 
In addition, we show that the mBNN achieves the optimal sparsity asymptotically.
We assume that $\mathcal{D}^{(n)} := \{ (\boldsymbol{X}_i , Y_i) \}_{i \in [n]}$ are independent copies following the true distribution $\mathbb{P}_0$ specified by either the model (\ref{reg}) or model (\ref{cla}).

\subsection{Posterior concentration rate for nonparametric regression}
We consider the nonparametric regression model (\ref{reg}).
We assume the true regression function $f_0$ belongs to the $\beta$-Hölder class $\mathcal{H}_d^\beta$.\footnote{Theoretical results for hierarchical composition functions are provided in Appendix \ref{App_comp}.}
Here, the $\beta$-Hölder class $\mathcal{H}_d^\beta$ is given as
\begin{equation*}
\mathcal{H}_d^\beta := \{ f : [-1,1]^d \to \mathbb{R} ; ||f||_{\mathcal{H}^\beta} < \infty \},
\end{equation*}
where $||f||_{\mathcal{H}^\beta}$ denotes the Hölder norm defined by
\begin{align*}
||f||_{\mathcal{H}^\beta} &:=  \sum_{\bm{\alpha} : |\bm{\alpha}|_1 <\beta} \left\|\partial^{\bm{\alpha}} f\right\|_{\infty} \\
&+\sum_{\bm{\alpha}:|\bm{\alpha}|_1 =\lfloor\beta\rfloor} \sup _{\underset{\bm{x}_1 \ne \bm{x}_2}{\bm{x}_1, \bm{x}_2 \in [-1,1]^d}} 
\frac{\left|\partial^{\bm{\alpha}} f(\bm{x}_1)-\partial^{\bm{\alpha}} f(\bm{x}_2)\right|}{|\bm{x}_1-\bm{x}_2|_{\infty}^{\beta-\lfloor\beta\rfloor}}.
\end{align*}
For inference, we consider the probabilistic model
\begin{equation*}
Y_i \stackrel{ind.}\sim N( f_{\bm{M}, \bm{\theta} [-F,F]}^{\operatorname{mDNN}}(\bm{X}_i), \sigma^2),  
\end{equation*}
where $f_{\bm{M}, \bm{\theta}}^{\operatorname{mDNN}}(\bm{X}_i)$ has the $(L_n, \bm{p}_n)$ architecture with $L_n$ and $\bm{p}_n$ given as
\begin{align}
    L_n :=& \left\lceil C_L \log n \right\rceil, \label{L_n} \\
	p_n :=& \left\lceil C_p \sqrt{n} \right\rceil, \nonumber \\
	\bm{p_n} :=& (d, p_n , \dots, p_n, 1)^{\top} \in \mathbb{N}^{L_n + 2} \label{p_n1}
\end{align}
for positive constants $C_L$ and $C_p$ that are defined in Lemma \ref{kohler2021rate}.
Then, the likelihood $\mathcal{L} (w | \mathcal{D}^{(n)})$ of $w := (\bm{M}, \bm{\theta}, \sigma^2)$ is expressed as
\begin{align*}
    (2 \pi \sigma^2 )^{-\frac{n}{2}}
    \exp \left(-\frac{\sum_{i=1}^n (Y_i -  f_{\bm{M}, \bm{\theta}\ [-F, F]}^{\operatorname{mDNN}}(\bm{X}_i))^2}{2 \sigma^2}\right), 
\end{align*}
and the corresponding posterior distribution is given as
\begin{align*}
    \Pi_n (w \mid \mathcal{D}^{(n)}) \propto
    \Pi_n (w) \mathcal{L}(w | \mathcal{D}^{(n)}), 
\end{align*}
where $\Pi_n$ is the prior defined on Section \ref{section_prior}.

In the following theorem, we show that
the mBNN model achieves the optimal (up to a logarithmic factor) posterior concentration rate to the true regression function. 

\begin{restatable}[Posterior Concentration of the mBNN for regression problems]{theorem}{thmone} \label{theorem1}
	Assume $f_0 \in \mathcal{H}_d^\beta$, $\beta < d$, and there exist $F>0$ and $\sigma^2_{max}>0$ such that $\|f_0\|_{\infty} \leq F$ and $\sigma^2_0 \leq \sigma^2_{max}$.
	Consider the mDNN model with the ($L_n, \bm{p}_n$) architecture, where $L_n$ and $\bm{p}_n$ are given in (\ref{L_n}) and (\ref{p_n1}).
	If we put the prior given as (\ref{prior_s}), (\ref{prior_m}) and (\ref{prior_theta}) over  $\bm{M}$, $\bm{\theta}$ and any prior on $\sigma^2$ whose density (with respect to Lebesgue measure) is positive on its support $(0,\sigma_{max}^2]$,
	the posterior distribution concentrates to the $f_0$ and $\sigma_0^2$ at the rate $\varepsilon_{n}=n^{-\beta /(2 \beta+d)} \log ^{\gamma}(n)$ for $\gamma > \frac{5}{2}$ in the sense that
	\begin{align*}
	\Pi_n \Big( (f, \sigma^2) : \  
	& || f - f_0 ||_{2, \mathrm{P}_{X}} \\ &+ |\sigma^2 - \sigma_0^2 | > M_n \varepsilon_{n}  \Bigm\vert \mathcal{D}^{(n)}\Big) \overset{\mathbb{P}_{0}^{n}}
 {\to} 0
	\end{align*}
	as $n \to \infty$ for any $M_n \to \infty$, where $\mathbb{P}_{0}^{n}$ is the probability measure of the training data $\mathcal{D}^{(n)}$. 
\end{restatable}

The convergence rate $n^{-\beta /(2 \beta+d)}$ is known to be minimax lower bound when estimating the $\beta$-Hölder smooth function \cite{tsybakov2009introduction}. Our concentration rate is near optimal up to a logarithmic factor and adaptive to the smoothness of the true model.  

\paragraph{Comparison with other works}
Similar convergence rates are derived in the non-bayesian theoretical deep learning literature \cite{schmidt2020nonparametric, kohler2021rate}.
However, the architectures considered in \citet{schmidt2020nonparametric} and \citet{kohler2021rate}
depend on the smoothness $\beta$ of the true regression function, which is rarely known in practice. 
In contrast, architecture and prior of the mBNN do not depend on the smoothness $\beta$, which makes the mBNN very attractive.

\citet{polson2018posterior}, \citet{cherief2020convergence} and \citet{bai2020efficient} also derive
near optimal concentration rates for edge-sparse BNNs.
Posterior computations for their smoothness-adaptive models, however, are almost impossible and inferential cost of edge-sparse BNNs could be large.
\citet{jantre2021layer} provides the posterior concentration rate for node-sparse BNN, but 
their result does not guarantee minimax optimality and is not adaptive to the smoothness of the true model.
Theorem \ref{theorem1} is the first result for theoretical optimality of the Bayesian analysis for node-sparse DNNs.


\subsection{Posterior concentration rate for binary classification}
Theoretical results for regression problem can be extended to classification problem. 
We consider the classification problem (\ref{cla}) with $K=2$, and denote $f_0 := \left(\bm{f}_0\right)_2 - \left(\bm{f}_0\right)_1.$
We assume that $f_0$ belongs to the $\beta$-Hölder class $\mathcal{H}_d^\beta$.
For inference, we consider the probabilistic model
\begin{align*}
Y_i \stackrel{ind.}\sim \operatorname{Bernoulli}(\phi \circ f_{\bm{M}, \bm{\theta} [-F, F]}^{\operatorname{mDNN}}(\bm{X}_i)),  
\end{align*} 
where $\phi$ is the sigmoid function and $f_{\bm{M}, \bm{\theta}}^{\operatorname{mDNN}}$ has the $(L_n, \bm{p}_n)$ architecture with $L_n$ and $\bm{p}_n$ given as (\ref{L_n}) and (\ref{p_n1}).
Then, the likelihood $\mathcal{L}(w | \mathcal{D}^{(n)})$ of $w := (\bm{M}, \bm{\theta})$ is expressed as
\begin{align*}
     \prod_{i=1}^n (\phi \circ f_{\bm{M}, \bm{\theta} [-F, F]}^{\operatorname{mDNN}}(\bm{X}_i))^{Y_i}
     (1-\phi \circ f_{\bm{M}, \bm{\theta} [-F, F]}^{\operatorname{mDNN}}(\bm{X}_i))^{1-Y_i}.
\end{align*}
In the following theorem, we prove that
the mBNN model achieves the optimal (up to a logarithmic factor) posterior concentration rate to the true conditional class probability adaptive to the smoothness of the true model.

\begin{restatable}[Posterior Concentration of the mBNN for classification problems]{theorem}{thm2} \label{theorem2}
	Assume $f_0 \in \mathcal{H}_d^\beta$, $\beta < d$, and there exists
    $F>0$ such that $\|f_0\|_{\infty} \leq F$.
	Consider the mDNN model with the ($L_n, \bm{p}_n$) architecture where $L_n$ and $\bm{p}_n$ are given in (\ref{L_n}) and (\ref{p_n1}).
	If we put the prior given as (\ref{prior_s}), (\ref{prior_m}) and (\ref{prior_theta}) over $\bm{M}$ and $\bm{\theta}$,
	the posterior distribution concentrates to the true conditional class probability at the rate $\varepsilon_{n}=n^{-\beta /(2 \beta+d)} \log ^{\gamma}(n)$ for $\gamma > \frac{5}{2}$ in the sense that
	\begin{align*}
	\Pi_n \Big( f : 
	||\phi \circ f - \phi \circ f_0 ||_{2, \mathrm{P}_{X}} > M_n \varepsilon_{n}  \Bigm\vert \mathcal{D}^{(n)}\Big) \overset{\mathbb{P}_{0}^{n}}{\to} 0
	\end{align*}
	as $n \to \infty$ for any $M_n \to \infty$, where $\mathbb{P}_{0}^{n}$ is the probability measure of the  training data $\mathcal{D}^{(n)}$.     
\end{restatable}

\subsection{Guaranteed sparsity of the mBNN}

Not only the fast concentration rate, the mBNN achieves the optimal sparsity too.
Theorem \ref{cor1}, which is a by-product of the proofs for Theorems \ref{theorem1} and \ref{theorem2},
gives the level of sparsity of the mBNN.

\begin{theorem}[Guaranteed sparsity of the mBNN] \label{cor1}
Under the assumptions in Theorem \ref{theorem1} or Theorem \ref{theorem2}, we have
\begin{align*}
	\Pi_n \Big( f : f= f_{\bm{M}, \bm{\theta} [-F,F]}^{\operatorname{mDNN}}, \ \max_{l \in [L]} |\bm{m}^{(l)}|_0 > \mathfrak{s}_n  \Bigm\vert \mathcal{D}^{(n)}\Big) \overset{\mathbb{P}_{0}^{n}}
 {\to} 0
\end{align*}
	as $n \to \infty$, where $\mathfrak{s}_n$ is defined as
\begin{align*}
\mathfrak{s}_n := \left\lceil C_p \left(n^\frac{d}{2\beta + d} (\log n) \right)^{1/2} \right\rceil
\end{align*}
and $\mathbb{P}_{0}^{n}$ is the probability measure of the training data $\mathcal{D}^{(n)}$. 
\end{theorem}

\citet{yarotsky2017error} proves that the lower bound of the sparsity 
of DNNs to approximate functions in $\mathcal{H}_d^\beta$ 
is equal to $\mathfrak{s}_n$ up to a logarithmic factor.
That is, the mBNN automatically learns the optimal sparsity from data.

\section{Posterior inference}
\label{sec5}


Let $\mathcal{D}^{(n)} := \{ (\boldsymbol{X}_i , Y_i) \}_{i \in [n]}$ be training data. 
Let $w$ denote all of the parameters in the mBNN, that is, 
$w := (\bm{M}, \bm{\theta}, \sigma^2)$ for the regression problem (\ref{reg}) and $w := (\bm{M}, \bm{\theta})$ for the classification problem (\ref{cla}). 
For a new test example $\bm{x} \in \mathcal{X}$, the prediction with the mBNN is done by the predictive distribution:
\begin{equation*}
p(y \mid \boldsymbol{x}, \mathcal{D}^{(n)})=\int_{w} p(y \mid \bm{x}, w) \Pi_n (w \mid \mathcal{D}^{(n)}) dw,
\end{equation*}
where $\Pi_n (w \mid \mathcal{D}^{(n)})$ is the posterior posterior distribution of $w$.
When the integral is difficult to be evaluated, it is common to approximate it by the Monte Carlo method
\begin{equation*}
p(y \mid \bm{x}, \mathcal{D}^{(n)}) \approx \frac{1}{T} \sum_{t=1}^T p(y \mid \bm{x}, w^{(t)}),
\end{equation*}
where $w^{(t)} \sim \Pi_n(w \mid \mathcal{D}^{(n)})$.
In this section, we develop a MCMC algorithm to sample $w$ efficiently from
$\Pi_n(w \mid \mathcal{D}^{(n)})$.

\subsection{MCMC algorithm} \label{sec_MCMC}

The proposed MCMC algorithm samples $\bm{\theta}$ (and $\sigma^2$) given $\bm{M}$ and $\mathcal{D}^{(n)}$ and then
samples $\bm{M}$ given $\bm{\theta}$ (and $\sigma^2$) and $\mathcal{D}^{(n)}$, and iterates these two samplings until convergence. 

There are various efficient sampling algorithms for $\bm{\theta}$ (and $\sigma^2$) given $\bm{M}$ and $\mathcal{D}^{(n)}$
such as Hamiltonian Monte Carlo (HMC) \cite{neal2011mcmc}, Stochastic Gradient Langevin Dynamics (SGLD) \cite{welling2011bayesian} and Stochastic Gradient HMC (SGHMC) \cite{chen2014stochastic}. 
In practice, we select a sampling algorithm among those 
depending on the sizes of data and model.

For generating $\bm{M}$ from its conditional posterior,
we consider the Metropolis-Hastings (MH) algorithm. 
The hardest part is to design a good proposal distribution since the dimension of $\bm{M}$ is quite large and all entries are binary. 
In the next subsection, we propose
an efficient proposal distribution for $\bm{M}.$

\subsection{Proposal  for the MH algorithm} \label{sec5.2}

Essentially, sampling $\bm{M}$ is equivalent to sampling a large dimensional binary vector. 
A well known strategy for sampling a large dimensional binary vector
is to use the MH algorithm with
the locally informed proposal \cite{umrigar1993accelerated, zanella2020informed} 
given as
\begin{align} 
    q\left(\bm{M}^{\star} \mid \bm{M}\right) 
    \propto e^{\frac{1}{2}\left( l(\bm{M}^{\star})-l(\bm{M})\right)} \mathbb{I}\left(\bm{M}^{\star} \in H(\bm{M})\right), \label{hamming}
\end{align}
where $l(\bm{M})$ is the log-posterior of $\bm{M}$ and $H(\bm{M})$ is the Hamming ball of a certain size around $\bm{M}$.
The key point of (\ref{hamming}) is to give more probability to $\bm{M}^{\star}$ whose posterior probability is high. 
While powerful, this locally informed proposal requires to compute $l(\bm{M}^{\star})$ for every $\bm{M}^{\star} \in H(\bm{M}),$ which is time consuming.
To resolve this problem, we propose a proposal distribution which only uses the information of the current $\bm{M}$.

First, we select either {\it birth} or {\it death}, where
{\it birth} makes some inactive nodes become active and 
{\it death} makes some active nodes become inactive.
When {\it death} is selected, motivated by the pruning algorithms \cite{lee2018snip, tanaka2020pruning}, our strategy is to prune less sensitive nodes. 
That is, we delete nodes which do not affect much to the current DNN model when their values are changed.
On the other hand, when {\it birth} is selected, we choose some of inactive nodes with equal probabilities
and make them active. 
We use equal probabilities because the posterior of the edges connected to the inactive nodes
is the same as the prior and thus there is no reason to prefer certain inactive nodes more. Moreover, the proposal with equal probabilities for {\it birth} is helpful for 
increasing the acceptance rate of {\it death} to result in fast mixing.

\begin{algorithm}[t] 
    \small
	\caption{The algorithm of the proposed MH algorithm} \label{algMH}
	\begin{algorithmic}[1]
		\STATE Sample $u \sim \operatorname{Bernoulli}(0.5)$,
  $N \sim \operatorname{Uniform}([N_{\operatorname{max}}])$.
		\STATE Calculate the selection probability $\bm{Q}_u$ using \\
  (\ref{birth}) or (\ref{death}).
		\STATE Sample $N$ many nodes $\{i_1 , \dots, i_N\}$ by
		\begin{equation*} 
		    \{ i_1 , \dots, i_N \} \sim \operatorname{Multi}_{\left( N, \bm{Q}_u \right)}. \label{our_proposal2}
		\end{equation*}		
		\STATE $\bm{M}^{\star} = \operatorname{flip}({\bm{M}^{\text{curr}}, \{i_1 , \dots, i_N\}}),$ where $\bm{M}^{\text{curr}}$ is the current masking vectors.
		\STATE Accept or reject $\bm{M}^{\star}$ with the acceptance probability (\ref{eq:accept}).
	\end{algorithmic}
\end{algorithm}

To be more specific, we first select the move $u \in \{0,1\}$ with probability 1/2, 
where $u=0$ and $u=1$ indicates {\it birth} and {\it death}, respectively.
In addition, we select an inteager $N$ from $[N_{\max}]$ uniformly, where $N_{\max}$ is a prespecified positive integer. Then, we select randomly $N$ nodes among the nodes whose masking values are $u$ 
following $\operatorname{Multi}_{\left( N, \bm{Q}_u \right)}$,  
where 
\begin{align} 
    \bm{Q}_0 \propto& \ \mathbb{I}(\bm{m}_j^{(l)} = 0), &&\text{({\it birth})} \label{birth}\\
    \bm{Q}_1 \propto& \ \mathbb{I}(\bm{m}_j^{(l)} = 1) \exp\left( -|\nabla l(\bm{M})|/2 \right). &&\text{({\it death})}  \label{death}   
\end{align} 
Even though $l(\bm{M})$ is defined only on binary vectors,
the gradient $\nabla l(\bm{M}) =\partial l(\bm{M})/\partial \bm{M}$
of $l(\bm{M})$ can be defined by extending the domain of $l(\bm{M})$ appropriately.
Finally, we flip the masking values of the selected nodes to have a new proposal $\bm{M}^{*}.$
To sum up, the proposal distribution first selects a set of nodes $\{ i_1 , \dots, i_N \}$ following $\operatorname{Multi}_{\left( N, \bm{Q}_u \right)}$ and changes their mask values to $(1-\bm{M}_{i_1} , \dots, 1-\bm{M}_{i_N})$. 
Then, we accept the new proposal $\bm{M}^{*}$ with probability
\begin{align}
\label{eq:accept}
\min \left( 1 , \frac{\Pi_n(\bm{M}^{\star}\mid X^{(n)}, Y^{(n)} , \bm{\theta})}{\Pi_n( \bm{M} \mid X^{(n)}, Y^{(n)}, \bm{\theta})}
\frac{q\left(\bm{M} \mid \bm{M}^{\star}\right)}{q\left(\bm{M}^{\star} \mid \bm{M}\right)}
 \right),
\end{align}
where
\begin{align*}
\frac{q\left(\bm{M} \mid \bm{M}^{\star}\right)}{q\left(\bm{M}^{\star} \mid \bm{M}\right)}
= \frac{\operatorname{Multi}_{( N, \bm{Q}^{\star}_{1-u} )}  (\{i_1 , \dots, i_N\})}{\operatorname{Multi}_{( N, \bm{Q}_u )}  (\{i_1 , \dots, i_N\})}
\end{align*}
and $\bm{Q}^{\star}_{1-u}$ is the selection probability vector defined by (\ref{birth}) or (\ref{death})
with the mask vectors $\bm{M}^{\star}.$
Note that some inactive nodes become active when $u=0$ (i.e. {\it birth} of nodes) and some active nodes become inactive when $u=1$ (i.e. {\it death} of nodes). 
The proposed MH algorithm is summarized in Algorithm \ref{algMH}. 

One may consider a linear approximation of $l$ in (\ref{hamming}) using the gradient information at $\bm{M}$, as is suggested by \citet{grathwohl2021oops} and \citet{zhang2022langevin}.
However, we found that the linear approximation
is not accurate for the mBNN, which is partly  because the corresponding DNN is not locally smooth enough.
In section \ref{sec6_5}, we provide the results of the experiment for comparing several proposal distributions
which support the choice of (\ref{birth}) and (\ref{death}).

\begin{figure}[t!]
\centering
\subfigure[BNN]{\includegraphics[width=0.45\linewidth]{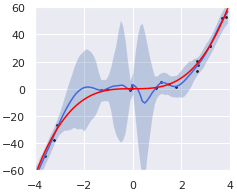}}
\hfill
\subfigure[mBNN]{\includegraphics[width=0.45\linewidth]{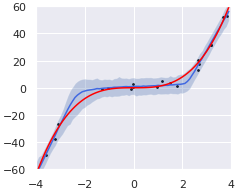}}
\caption{\textbf{Simulated data.}
Predictive distributions of BNN and the mBNN on simulated data.
The true polynomial function and the noisy observations are plotted by the red line and black dots, respectively. For each predictive distribution,
the mean  is drawn by the blue line and
the 95\% predictive interval is shown as the shaded areas.
} \label{poly_simul}
 \vskip -0.2in
\end{figure}

\section{Experiment} \label{sec6}

In this section, we perform experiments to empirically justify the usefulness of the mBNN. 
In Section \ref{sec6_1} and \ref{sec6_2}, we conduct an experiment to demonstrate the necessity of masking variables using simulation and real datasets. 
In Section \ref{sec6_3}, we apply the mBNN to a Bayesian structural time series model
to illustrate its practical usefulness.
In Section \ref{sec6_4}, we extend the mBNN to CNN and experimentally show it 
is also useful for compressing large complex DNNs.
In Section \ref{sec6_5}, we investigate the efficiency of the proposal distribution (\ref{birth}) and (\ref{death}) in the MH algorithm.
All the experimental details as well as the results of additional 
numerical experiments are given in Appendix \ref{App_setting} and \ref{App_additional}.
The code is available at \href{https://github.com/ggong369/mBNN}{https://github.com/ggong369/mBNN}.

\subsection{Simulation} \label{sec6_1}
We obtain the predictive distribution when the true regression function is the noisy polynomial regression problem considered in \citet{hernandez2015probabilistic}.
Inputs $x_i$ are sampled from $x_i \sim \operatorname{Uniform}(-4,4)$ and the corresponding outputs $y_i$ are obtained by $y_i = x_i^3 + \epsilon_i$, where $\epsilon_i \sim N(0,9)$.
We generate 20 training examples and compare the predictive distribution of the mBNN with that of BNN by generating 1000 MCMC samples from each model.
For each model, the two hidden layer MLP with the layer sizes (1000,1000) is used.

The mean and 95\% predictive intervals of the predictive distributions of BNN and the mBNN
are presented in Figure \ref{poly_simul}. 
The figure illustrates that BNN quantifies uncertainty on a data sparse region overly to have a too wide predictive interval. 
Note that the true regression function is assumed to smooth and hence it is possible to transfer information on data dense regions to data sparse regions. 
Thus, proper uncertainty quantification even on data sparse regions could be possible. 
The coverage probability of the predictive interval of the mBNN
is 95.3\%, which is obtained by
generating additional 1000 test sample, while that of BNN is 98.6\%.
The results indicate that deleting unnecessary nodes is important not only for fast convergence rates but also proper uncertainty quantification.

\begin{table}[t!] 
  \caption{\textbf{Real dataset.}
  Performance comparison of BNN, NS-VI and the mBNN on UCI datasets.
  The boldface numbers are the best one among the three methods.}
  \begin{adjustbox}{center,max width=\linewidth}
  \setlength{\aboverulesep}{0.2pt}
  \setlength{\belowrulesep}{0.2pt}
  \scalebox{0.75}{\small
    \begin{tabular}{c|c|cccc}
      \hline
       \textbf{Dataset} & \textbf{Method} & \textbf{Coverage} & \textbf{RMSE} & \textbf{NLL} & \textbf{CRPS} \\
    \hline
    \multirow{3}{*}{Boston} 
    &  BNN   & 0.912(0.006)  & 3.411(0.145) & 2.726(0.087) & 1.738(0.052)   \\
    &  NS-VI & 0.746(0.010)  & 3.079(0.179) & 4.242(0.568) & 1.661(0.069)   \\
      &  mBNN & \textbf{0.933}(0.007)  & \textbf{2.902}(0.143) & \textbf{2.472}(0.085) &  \textbf{1.462}(0.055)  \\
    \hline
      \multirow{3}{*}{Concrete}   
      & BNN & 0.908(0.008) & 5.080(0.178) & 3.091(0.055) & 2.658(0.072)  \\
    &  NS-VI & 0.568(0.011) & 5.046(0.149) & 9.713(0.686) & 2.965(0.088)    \\
        & mBNN	& \textbf{0.912}(0.009) & \textbf{4.913}(0.180) & \textbf{3.027}(0.046) & \textbf{2.628}(0.087)  \\
    \hline
     \multirow{3}{*}{Energy}
     & BNN	& \textbf{0.945}(0.005) & 0.591(0.017) & 0.902(0.025) & 0.322(0.007)  \\
    &  NS-VI  & 0.913(0.006)  & 1.322(0.117) & 1.792(0.121) & 0.720(0.055)  \\
      & mBNN	& \textbf{0.945}(0.004) & \textbf{0.474}(0.015) & \textbf{0.670}(0.034) & \textbf{0.256}(0.006)  \\
     \hline
     \multirow{3}{*}{Yacht}
     & BNN & 0.977(0.006) & 0.675(0.048) & 1.011(0.056) & 0.332(0.013) \\
    &  NS-VI & 0.932(0.011)  & 1.842(0.096) & 1.915(0.063) & 0.969(0.042)   \\
      & mBNN & \textbf{0.953}(0.008) & \textbf{0.664}(0.044) & \textbf{0.932}(0.062) & \textbf{0.318}(0.015)  \\
    \hline
    \end{tabular}
    }
  \end{adjustbox}
 \vskip -0.2in \label{table1}
\end{table}

\subsection{Real dataset}\label{sec6_2}

We evaluate BNN, node-sparse VI (NS-VI) \cite{louizos2017bayesian} and the mBNN on four UCI regression datasets (\textit{Boston, Concrete, Energy, Yacht}). 
For each dataset, we construct 20 random 90-to-10 train-test splits to provide the standard errors.
For each method, the two hidden layer MLP with the layer sizes (1000,1000) is used. 
We select 20 models from several thousands MCMC samples 
and compare the predictive distributions obtained by the selected 20 models
in terms of generalization and uncertainty quantification.
 
In Table \ref{table1}, we report the means and standard errors of the performance measures 
on the 20 repeated experiments. For the performance measures, 
the coverage probability of the 95\% predictive interval (Coverage),
the root mean square error (RMSE)
of the Bayes estimator, the negative log-likelihood (NLL)  
and continuous ranked probability score (CRPS) \cite{gneiting2007strictly}
on test data are considered.


It is obvious that the mBNN outperform the other two competitors with respect to all of the four measures. That is, the mBNN is good at not only estimating the regression function and but also
uncertainty quantification. It is noticeable that NS-VI performs too badly, which suggests that
full Bayesian analysis of DNN is must.
More detailed results such as the sparsity and the sensitivity to the hyper-parameter selection are provided in Appendix \ref{App_additional}.

\begin{figure}[t]
\centering
\subfigure[Linear ; (RMSE, NLL) = (1.34, 2.36)]{\includegraphics[width=0.8\linewidth]{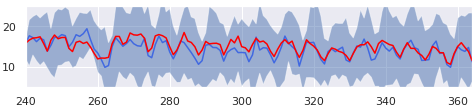}}
\subfigure[BNN ; (RMSE, NLL) = (1.59, 2.10)]{\includegraphics[width=0.8\linewidth]{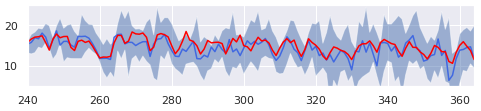}}
\subfigure[mBNN ; (RMSE, NLL) = (1.21, 1.64)]{\includegraphics[width=0.8\linewidth]{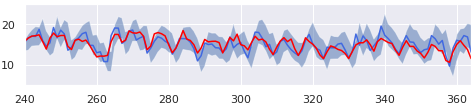}}
\caption{\textbf{Daily search volume dataset.} 
Predictive distribution of BSTS model using linear, DNN and mDNN.  
X-axis and y-axis represent time (in day) and search volumes (in thousand), respectively.
For each method, the means of predictive distributions and 95\% predictive intervals are shown as the blue line and shaded area, respectively.
The observed values are plotted by the red line.
} \label{image_BSTS}
\end{figure}

\subsection{Application to the Bayesian structural time series model}\label{sec6_3}
Bayesian structural time series (BSTS) models \cite{scott2014predicting, qiu2018multivariate}
provide a useful tool for time series forecasting, nowcasting, inferring causal relationships and anomaly detection \cite{brodersen2015inferring, feng2021time}.
The key of BSTS is the state space model, which is given as
\begin{align*}
y_t &= \mu_t + \bm{\beta}^{\top} \bm{x}_t + \epsilon_t , && \epsilon_t \sim \mathcal{N}(0,\sigma_{\epsilon}^2)\\
\mu_{t+1} &= r \mu_{t} + \eta_t  , && \eta_t \sim \mathcal{N}(0,\sigma_{\eta}^2)\\
\mu_0 &\sim N(a_0, \sigma_0^2)
\end{align*}
where $\bm{x}_t \in \mathbb{R}^d$, $y_t \in \mathbb{R}$ and $\mu_t \in \mathbb{R}$ denote observed input variables, output variable and unobserved local trend at time $t$, respectively.
For inference and prediction, MCMC algorithm using Kalman filter \cite{welch1995introduction, durbin2002simple} is mainly used.

In many cases, the linear component
$\bm{\beta}^{\top} \bm{x}_t$ may not be insufficient to explain complicate relations between inputs and outputs, and  DNN can be considered instead. In turn, the mBNN is a useful inferential tool for this nonlinear BSTS. To illustrate that the mBNN works well for BSTS, we analyze a real dataset consisting of
daily search volumes of several keywords collected by a Korea search platform company.
The dataset consists of daily search volumes of keywords in year 2021 associated with 
a pre-specified product (eg. shampoo). The aim is to predict the search volume of the pre-specified
product based on the search volumes of other related keywords. 

We apply the three BSTS models corresponding to the three regression components - linear, BNN and the mBNN. For BNN and the mBNN, the two hidden layer MLP with the layer sizes (100,100) is used.
The predictive intervals with the NLL and RMSE values on the data from $t=241$ to $t=365$
obtained by the posterior distribution inferred on the data from $t=1$ to $t=240$
are presented in figure \ref{image_BSTS}. It is clearly observed that the mBNN
is superior in both nowcasting ability and uncertainty quantification.
It is interesting that the predictive intervals of the linear and BNN models
are much wider than those of the mBNN, which amply indicates that the choice of an appropriate
architecture of DNN is crucial for desirable uncertainty quantification.
More details about the dataset, methodology and additional experimental results are provided in Appendix \ref{appendix_bsts}.

\begin{table}[t] 
  \caption{\textbf{Image datasets.} Performance comparison
  of NS-VI, NS-Ens, NS-MC and mBCNN on image datasets.
  When inferring the posterior of mBCNN, we use SGLD \cite{welling2011bayesian}.
  } 
  \begin{adjustbox}{center,max width=\linewidth}
  \setlength{\aboverulesep}{0.2pt}
  \setlength{\belowrulesep}{0.2pt}
  \scalebox{0.75}{\small
    \begin{tabular}{c|c|cccc}
      \hline
        & \textbf{Measure}   & \textbf{NS-VI} &  \textbf{NS-Ens} & \textbf{NS-MC} & \textbf{mBCNN}\checkmark  \\
    \hline
    \multirow{5}{*}{\rotatebox{90}{CIFAR10}}  
    &  ACC  & 0.906(0.002) & 0.926(0.003) & 0.918(0.002) & \textbf{0.932}(0.001) \\
      &  NLL & 0.298(0.004) & 0.255(0.011) & 0.525(0.006) & \textbf{0.220}(0.005) \\
      &  ECE & 0.009(0.001) & 0.012(0.002) & 0.020(0.001) & \textbf{0.008}(0.001) \\
      & FLOPs & 44.85(2.67)\% & 22.45(1.64)\% & 41.16(0.00)\% & \textbf{12.26}(0.06)\% \\
        &  Capacity & 23.24(1.56)\% & 12.78(0.75)\% & 41.07(0.00)\% & \textbf{3.47}(0.03)\% \\
    \cmidrule(r){1-6}
      \multirow{5}{*}{\rotatebox{90}{CIFAR100}}   
      & ACC  & 0.600(0.004) & 0.735(0.003) & 0.679(0.002) & \textbf{0.737}(0.001) \\
      &  NLL & 1.977(0.039) & 1.076(0.014) & 2.792(0.061) & \textbf{1.004}(0.008) \\
        & ECE & 0.006(0.000) & \textbf{0.002}(0.000) & 0.006(0.000) & \textbf{0.002}(0.000) \\
      & FLOPs & 37.38(0.64)\% & 27.02(2.30)\% & 53.50(0.00)\% & \textbf{18.34}(0.27)\% \\
        &  Capacity & 41.61(4.04)\% & 20.08(1.15)\% & 53.45(0.00)\% & \textbf{12.21}(0.07)\% \\
    \hline
    \end{tabular}
    }
  \end{adjustbox}
 \vskip -0.1in \label{table_image}
\end{table}

\subsection{Extension to convolution neural network} \label{sec6_4}

We also extend the mBNN to masked Bayesian convolution neural network (mBCNN) for image dataset.
We construct a masked CNN (mCNN) by adding masking vectors to the CNN model.
See Appendix \ref{appB_2} for details of the mBCNN.

We evaluate compression ability of the mBCNN on image datasets. 
For comparison, we use node-sparse versions of approximated Bayesian
methodologies : NS-VI, node-sparse Deep ensemble (NS-Ens) and node-sparse MC-dropout (NS-MC).
Detail description of competitors are provided in Appendix \ref{app_image}.
For each methods, ResNet18 \cite{he2016deep} architecture is used.
For fair comparison, we use five networks for inference in all the methods. 
We repeat the experiments 3 times with different seeds.

In Table \ref{table_image}, we report the means and standard errors of the performance measures on the repeated experiments.
For the performance measures, accuracy of the Bayes estimator, NLL, expected calibration error (ECE) \cite{kumar2019verified} of test data are considered.
In addition, we report inferential costs (FLOPs) and the numbers of nonzero parameters (Capacity) relative to the non-sparse model.
The results in Table \ref{table_image} show that mBCNN provides better generalization and uncertainty quantification with less inferential cost and model capacity compared to the other competitors. That is, the mBNN is a useful tool for compressing 
complex DNNs without hampering uncertainty quantification much.

\begin{figure}[t!]
\centering
\subfigure[Yacht]{\includegraphics[width=0.32\linewidth]{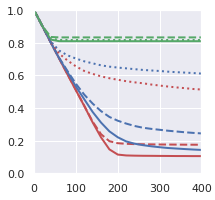}}
\hfill
\subfigure[CIFAR10]{\includegraphics[width=0.32\linewidth]{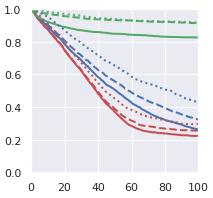}}
\hfill
\subfigure[CIFAR100]{\includegraphics[width=0.32\linewidth]{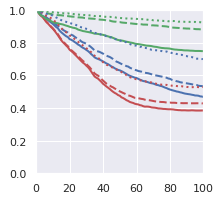}}
\caption{\textbf{Efficiency of the proposal.} 
The ratios of activated nodes (y-axis) relative to the largest DNN
are drawn as the MCMC algorithms iterate in the burn-in period (x-axis).
The solid, dashed and dotted lines correspond to (\ref{pro_uni}), (\ref{pro_reci}) and (\ref{pro_line}) for {\it birth}, respectively while blue, red and green lines correspond to (\ref{pro_uni}), (\ref{pro_reci}) and (\ref{pro_line}) for {\it death}, respectively. 
It is obvious that the number of activated nodes with  our proposal distribution (the solid red line) decreases much faster than the other proposals.
} \label{proposal_compare}
 \vskip -0.2in
\end{figure}

\subsection{Ablation study: Efficiency of the proposal}\label{sec6_5}
To illustrate the efficiency of 
our proposal distribution in Algorithm \ref{algMH}, we conduct a comparative experiment.
As an alternative to the proposal distribution with selection probability (\ref{birth}) for {\it birth} and (\ref{death}) for {\it death}, we consider the following candidates for the selection probability
of either {\it birth} or {\it death}: 
\begingroup\makeatletter\def\f@size{9}\check@mathfonts
\def\maketag@@@#1{\hbox{\m@th\normalfont#1}}%
\begin{align} 
    \bm{Q}_u \propto& \ \mathbb{I}(\bm{m}_j^{(l)} = u), \label{pro_uni}\\
    \bm{Q}_u \propto& \ \mathbb{I}(\bm{m}_j^{(l)} = u) \exp\left( -|\nabla l(\bm{M})|/2 \right),  \label{pro_reci} \\   
    \bm{Q}_u \propto& \ \mathbb{I}(\bm{m}_j^{(l)} = u) \exp\left( (1-2u) \nabla l(\bm{M})/2 \right). \label{pro_line} 
\end{align} \endgroup
While (\ref{pro_uni}) gives the same selection probability on each active (or inactive) node, 
(\ref{pro_reci}) and (\ref{pro_line}) put proposals depending on the role of each node.
The proposal of (\ref{pro_reci}) is devised to select less sensitive nodes more by giving proposal probabilities 
reciprocally proportional to $|\nabla l(\bm{M})|$, and (\ref{pro_line}) uses a linear approximation of $l$ in (\ref{hamming}) \cite{grathwohl2021oops}.

We compare the speeds of convergence of the MCMC algorithm with
the 9 proposal distributions which are all $3^2$ combinations 
of the three candidates
 (\ref{pro_uni}), (\ref{pro_reci}) and (\ref{pro_line})  and the two moves for {\it birth} and {\it death}.
Figure \ref{proposal_compare} presents how the ratios of the activated nodes 
relative to the largest DNN
decrease in the burn-in phase of the MCMC algorithm.
It is obvious that our proposal (the red solid lines)
is most fast to eliminate unnecessary nodes.
Note that the proposals involving the linear approximation of $l(\bm{M})$  do not work well, which 
suggests that DNNs are not smooth enough to be approximated linearly.

\section{Discussion}
\label{sec7}
We have proposed the mBNN which 
searches for a DNN with an appropriate complexity. 
We prove theoretical optimalities of the mBNN and develop an efficient MCMC algorithm.
By extensive numerical studies, we illustrate that the proposed BNN 
discovers well condensed DNN architectures with better prediction accuracy and uncertainty quantification compared to large DNNs. 

A node-sparse network can be considered as a dense network with a smaller width because all edges connected to survived nodes are active. This property is sharply contrast with edge-sparse networks. A key difference of the node-sparse network and a dense network with a smaller width is that the widths of each layer can vary for node-sparse networks while they should be fixed in advance for dense network. In practice, it would be difficult to find the optimal width in advance before analyzing data. mBNN is a kind of tools to find the optimal width data adaptively.

We do not insist that the proposal distribution in our MCMC algorithm is optimal.
There would be more efficient proposals. More data adaptive proposals for {\it birth}
would be possible which we leave as a future work.

Condensing the posterior distribution is also interesting. 
We have to employ multiple DNNs in the prediction phase, which would be expensive.
Summarizing multiple posterior samples into a single random DNN would be useful.
For bootstrapping, \citet{shin2021neural} proposes a similar method, which can be modified for the mBNN.

\section*{Acknowledgements}
This work was supported by National Research Foundation of Korea(NRF) grant funded by the Korea government (MSIT) (No. 2020R1A2C3A0100355014), Institute of Information \& communications Technology Planning \& Evaluation (IITP) grant funded by the Korea government(MSIT) 
[NO.2022-0-00184, Development and Study of AI Technologies to Inexpensively Conform to Evolving Policy on Ethics ] and INHA UNIVERSITY Research Grant. 



\nocite{langley00}

\bibliography{mBNN.bib}
\bibliographystyle{icml2023}

\newpage
\appendix
\onecolumn

\newpage
\numberwithin{equation}{section}
\section{Proofs for Main Theorems}\label{appA}
\renewcommand{\theequation}{A.\arabic{equation}}

\subsection{Additional notations}\label{appA_1}
In this section, we describe additional notations not mentioned earlier.

We define $\mathcal{F}^{\operatorname{DNN}}(L, \bm{p}, F)$ as the function class of truncated DNNs with $(L, \bm{p})$ architecture :
\begin{align*}
\mathcal{F}^{\operatorname{DNN}}(L, \bm{p}, F) := \big\{ f : f = f_{\bm{\theta} [-F, F]}^{\operatorname{DNN}} \text{ is a DNN with }(L, \bm{p}) \quad 
\\ \text{ architecture } \text{truncated on } [-F, F]\big\}.
\end{align*}
In a similar fashion, we define $\mathcal{F}^{\operatorname{mDNN}}(L, \bm{p}, F)$ as the function class of truncated mDNNs with $(L, \bm{p})$ architecture :
\begin{align*}
\mathcal{F}^{\operatorname{mDNN}}(L, \bm{p}, F) := \big\{ f : f = f_{\bm{M}, \bm{\theta} [-F,F]}^{\operatorname{mDNN}} \text{ is a mDNN with }(L, \bm{p}) \quad 
\\ \text{ architecture } \text{truncated on } [-F, F]\big\}.
\end{align*}
and $\mathcal{F}^{\operatorname{mDNN}}(L, \bm{p}, F, s)$ as
\begin{align*}
\mathcal{F}^{\operatorname{mDNN}}(L, \bm{p}, F, s) := \big\{& f : f = f_{\bm{M}, \bm{\theta} [-F,F]}^{\operatorname{mDNN}} \in \mathcal{F}^{\operatorname{mDNN}}(L, \bm{p}, F),\ \max_{l \in [L]} |\bm{m}^{(l)}|_0 \leq s \big\}.
\end{align*}

For a real number $x \in \mathbb{R}$, we denote $\left\lceil x \right\rceil := \min\{ z \in \mathbb{Z} : z \geq x \}$.
For a vector $\bm{x} \in \mathbb{R}^d$ and $m$ dimensional index vector $\bm{j} \subset [d]$, we denote $(\bm{x})_{\bm{j}} \in \mathbb{R}^m$ as the sub-vector whose elements are consist of $\bm{j}$th index of $\bm{x}$.
Let $\bm{x}^{(n)} := (\bm{x}_i)_{i=1}^n$, $\bm{X}^{(n)} := (\bm{X}_i)_{i=1}^n$.
For a real-valued function  $f : \mathcal{X} \to \mathbb{R}$  and $1 \leq p < \infty$, we denote $||f||_{p,n} := (\sum_{i=1}^{n} f(\bm{x}_i)^p/n)^{1/p}$ and $||f||_{p,\mathrm{P}_{\bm{X}}} := \left(\int_{\bm{X} \in \mathcal{X}} f(\bm{X})^p d\mathrm{P}_{\bm{X}} \right)^{1/p}$ where
$\mathrm{P}_{\bm{X}}$ is a probability measure defined on input space $\mathcal{X}$.

For two positive sequences $\{a_n\}$ and $\{b_n\}$, we denote $a_n \lesssim b_n$ if there exists a positive sequence $C>0$ such that $a_n \leq Cb_n$ for all $n \in \mathbb{N}$. 
We denote $a_n \asymp b_n$ if $a_n \lesssim b_n$ and $a_n \gtrsim b_n$ hold.
We use the little $o$ notation, that is, we write $a_n = o(b_n)$ if $\lim_{n \to \infty} a_n / b_n = 0$.

Let $\mathcal{F}$ be a set of functions $\mathcal{X} \to \mathbb{R}$ and $d_n$ be a semimetric defined on $\mathcal{F}$. We denote $\mathcal{N}(\varepsilon, \mathcal{F}, d_n)$ and $\mathcal{M}(\varepsilon, \mathcal{F}, d_n)$ as the $\varepsilon$-covering number and $\varepsilon$-packing number of $\mathcal{F}$ w.r.t. $d_n$, respectively.
Also, we denote  $V_{\mathcal{F}}^{+}$ as the VC dimension of the set $\mathcal{F}^{+}:=\left\{\left\{(x, t) \in \mathcal{X} \times \mathbb{R} ; t \leq f(x)\right\} ; f \in \mathcal{F}\right\}$. 

\subsection{Auxiliary lemmas}\label{appA_2}
First, we describe a lemma that approximates Hölder smooth functions as DNN functions.
\begin{lemma}[Theorem 2 of \citet{kohler2021rate}]\label{kohler2021rate}	
	There exists $C_L>0$ and $C_p>0$ only depending on $d$ such that for every $f_0 \in \mathcal{H}_d^\beta$ with $\left\|f_0\right\|_{\infty} \leq F$,
	there exist $f_{\hat{\bm{\theta}} \ [-F, F]}^{\operatorname{DNN}} \in \mathcal{F}^{\operatorname{DNN}}(L_n, \bm{v}_n, F)$ with
	\begin{align*}
	L_n :=& \left\lceil C_L \log n \right\rceil, \\
	v_n :=& \left\lceil C_p \left(n^\frac{d}{2\beta + d} (\log n)^{-1}\right)^{1/2} \right\rceil, \\
	\bm{v}_n :=& (d, v_n, ... , v_n, 1)^{\top} \in \mathbb{N}^{L_n + 2},
	\end{align*}
	such that
	\begin{align*}
	\left\|f_{\hat{\bm{\theta}} \ [-F, F]}^{\operatorname{DNN}}-f_{0}\right\|_{\infty} \lesssim n^{-\frac{\beta}{2\beta + d}} \log n 
	\end{align*}
	and
	\begin{align} 
    |\hat{\bm{\theta}}|_{\infty} \leq n \label{upper_param}
	\end{align}
	hold.
\end{lemma}
Note that the upper bound (\ref{upper_param}) is not mentioned in statement of \citet{kohler2021rate}, but it can be easily confirmed by following their proof. Next, we describe a standard tool for establishing concentration rates.

\begin{lemma}[Theorem 4 of \citet{ghosal2007convergence}]\label{ghosal2007convergence}	
	Let $\left(\mathfrak{Y}^{(n)}, \mathcal{A}^{(n)}, P_{\eta}^{(n)}: \eta \in \mathcal{F}_n \right)$ be a sequence statistical experiments with observations $Y^{(n)}$. 
	We consider the case where the observation $Y^{(n)}$ is a vector $Y^{(n)} = (Y_1 , \dots, Y_n)^{\top}$ of independent observations $Y_i$. 
	We assume that the distribution $P_{\eta , i}$ of the $i$th component $Y_i$ possesses a density $p_{\eta, i}$ relative to Lebesgue measure for $i \in [n]$.
	We define
	\begin{align*}
	K_i (\eta_0, \eta) =& \int \log (p_{\eta_0, i} / p_{\eta, i}) dP_{\eta_0, i},\\
	V_{2,0;i} (\eta_0, \eta) =& \int  \left( \log (p_{\eta_0, i} / p_{\eta, i}) - K_i (\eta_0, \eta)\right)^2 dP_{\eta_0, i}, 
	\end{align*}
	and
	\begin{align*}
	B_{n}^{*}\left(\eta_{0}, \varepsilon_n ; 2\right)=\Big\{\eta \in \mathcal{F}_n:  \frac{1}{n}\sum_{i=1}^n K_i \left(\eta_{0}, \eta\right) \leq \varepsilon_n^{2}, 
	\ \frac{1}{n}\sum_{i=1}^n V_{2, 0 ; i}\left(\eta_{0}, \eta\right) \leq \varepsilon_n^{2}\Big\}.
	\end{align*}
	Let $h_n$ be a semimetric on $\mathcal{F}_n$ with the property that there exist universal constants $\xi>0$ and $K>0$ such that for every $\varepsilon > 0$ and for each $\eta_1 \in \mathcal{F}_n$ with $h_n(\eta_1, \eta_0)>\varepsilon$, there exists a test $\phi_n$ such that
	\begin{gather}
	P_{\eta_{0}}^{(n)} \phi_{n} \leq e^{-K n \varepsilon^{2}},
	\sup _{\eta_2 \in \mathcal{F}_n: h_{n}\left(\eta_2, \eta_{1}\right)<\varepsilon \xi} P_{\eta_2}^{(n)}\left(1-\phi_{n}\right) \leq e^{-K n \varepsilon^{2}}. \nonumber
	\end{gather}	
	Let $\varepsilon_{n}>0, \varepsilon_{n} \to 0$ and $(n\varepsilon_{n}^2)^{-1} = O(1)$.
	If for every sufficiently large $j \in \mathbb{N}$,
	\begin{gather}
	\sup _{\varepsilon>\varepsilon_{n}} \log N\left(\frac{1}{2} \varepsilon \xi,\left\{\eta \in \mathcal{F}_{n}: h_{n}\left(\eta, \eta_{0}\right)<\varepsilon\right\}, h_{n}\right) \leq n \varepsilon_{n}^{2}, \nonumber \\
	\frac{\Pi_{n}\left(\eta \in \mathcal{F}_{n}: j \varepsilon_{n}<h_{n}\left(\eta, \eta_{0}\right) \leq 2 j \varepsilon_{n}\right)}{\Pi_{n}\left(B_{n}\left(\eta_{0}, \varepsilon_{n} ; 2\right)\right)} \leq e^{K n \varepsilon_{n}^{2} j^{2} / 2} \nonumber
	\end{gather}
	for all but finite many $n$, then we have that
	\begin{equation*}
	P_{\eta_{0}}^{(n)} \Pi_{n}\left(\eta \in \mathcal{F}_n: h_{n}\left(\eta, \eta_{0}\right) \geq M_{n} \varepsilon_{n} \mid X^{(n)}\right) \rightarrow 0
	\end{equation*}
	for every $M_n \to \infty$.
\end{lemma}

Next, we state the lemma which describes an upper bound of supremum norm distance of two DNN functions whose parameters are similar.
\begin{lemma} \label{lemma_similar}
	Consider two DNN models $f_{\bm{\theta}_1}^{\operatorname{DNN}} : [-1,1]^d \to \mathbb{R} , f_{\bm{\theta}_2}^{\operatorname{DNN}} : [-1,1]^d \to \mathbb{R}$ with $(L , \bm{p})$ architecture, where $L\in \mathbb{N}$ and $\bm{p} = (d, p, p, ... , p, 1)^{\top} \in \mathbb{N}^{L+2}$ for some $p \in \mathbb{N}$. If $|\bm{\theta}_1|_{\infty} \leq B$, $|\bm{\theta}_2|_{\infty} \leq B$ and $|\bm{\theta}_1 - \bm{\theta}_2|_{\infty} \leq \delta$ holds for some $B>0$ and $\delta>0$, then 
	\begin{equation*}
	\|f_{\bm{\theta}_1}^{\operatorname{DNN}} - f_{\bm{\theta}_2}^{\operatorname{DNN}}\|_{\infty} \leq d p^{L} B^{L+1} (L+1)\delta
	\end{equation*}
	holds.
\end{lemma}	
\begin{proof}[proof of Lemma \ref{lemma_similar}]
	Define
	\begin{align*}
	f_{\bm{\theta}_1}^{\operatorname{DNN}}(\cdot)=  A_{L+1 , 1} \circ \rho \circ A_{L,1} \dots \circ \rho \circ A_{1,1} (\cdot),\\
	f_{\bm{\theta}_2}^{\operatorname{DNN}}(\cdot)=  A_{L+1 , 2} \circ \rho \circ A_{L,2} \dots \circ \rho \circ A_{1,2} (\cdot).
	\end{align*}
	Also, we define $\bm{h}_{\bm{\theta}, L^{\prime}} : [-1,1]^d \to \mathbb{R}^{p_{L^{\prime}}}$ for $L^{\prime} \in [L]$ as the DNN model whose output is $L^{\prime}$-th hidden layer of $f_{\bm{\theta}}^{\operatorname{DNN}}$. 
	In other words,    
	\begin{align*}
	\bm{h}_{\bm{\theta}_1, L^{\prime}}(\cdot)=  A_{L^{\prime} , 1} \circ \rho \circ A_{L^{\prime}-1,1} \dots \circ \rho \circ A_{1,1} (\cdot), \\
	\bm{h}_{\bm{\theta}_2, L^{\prime}}(\cdot)=  A_{L^{\prime} , 2} \circ \rho \circ A_{L^{\prime}-1,2} \dots \circ \rho \circ A_{1,2} (\cdot).
	\end{align*}
	We let $\bm{h}_{\bm{\theta}, L+1} (\cdot) = f_{\bm{\theta}}^{\operatorname{DNN}} (\cdot)$.
	Since
	\begin{align*}
	    \left\| |\bm{h}_{\bm{\theta}_1, L^{\prime} +1} - \bm{h}_{\bm{\theta}_2, L^{\prime} +1} |_{\infty}\right\|_{\infty} \leq
	    p |\bm{\theta}_1 - \bm{\theta}_2|_{\infty} 
	    \left\| |\bm{h}_{\bm{\theta}_1, L^{\prime}}|_{\infty}\right\|_{\infty}
	    + p|\bm{\theta}_2|_{\infty}
	    \left\| |\bm{h}_{\bm{\theta}_1, L^{\prime}} - \bm{h}_{\bm{\theta}_2, L^{\prime}}|_{\infty}\right\|_{\infty}
	\end{align*}
	and
	\begin{equation*}
	\left\| |\bm{h}_{\bm{\theta_1}, L^{\prime}}|_{\infty}\right\|_{\infty} \leq d p^{L^{\prime}-1} B^{L^{\prime}}
	\end{equation*}
	hold, we can show that $|\bm{\theta}_1 - \bm{\theta}_2|_{\infty} \leq \delta$ implies
	\begin{equation*}
	\left\| |\bm{h}_{\bm{\theta}_1, L^{\prime}}(x) - \bm{h}_{\bm{\theta}_2, L^{\prime}}(x)|_{\infty}\right\|_{\infty} \leq  
	d p^{L^{\prime}-1} B^{L^{\prime}} L^{\prime}\delta
	\end{equation*}
	for every $L^{\prime} \in [L+1]$ by recursion.
\end{proof}

Lastly, we state the lemma about empirical process theory.
\begin{lemma}[Theorem 19.3 of \citet{gyorfi2002distribution}]\label{gyorfi2002distribution}
	Let $\boldsymbol{X}, \boldsymbol{X}_1, \dots, \boldsymbol{X}_n$ be independent and identically distributed random vectors with values in $\mathbb{R}^d$. Let $K_1, K_2 \geq 1$ be constants and let $\mathcal{G}$ be a class of functions $g : \mathbb{R}^d \to \mathbb{R}$ with
	\begin{equation*}
	|g(\boldsymbol{x})| \leq K_1, \quad \mathbb{E}(g(\boldsymbol{X})^2) \leq K_2 \mathbb{E}(g(\boldsymbol{X})).
	\end{equation*}
	Let $0<\kappa<1$ and $\alpha>0$. Assume that
	\begin{equation*}
	\sqrt{n} \kappa \sqrt{1-\kappa} \sqrt{\alpha} \geq 288 \max \left\{2 K_{1}, \sqrt{2 K_{2}}\right\}
	\end{equation*}
	and that, for all $\boldsymbol{x}_1 , \dots , \boldsymbol{x}_n \in \mathbb{R}^d$ and for all $t \geq \frac{\alpha}{8}$,
	\begin{align*}
	\frac{\sqrt{n} \kappa(1-\kappa) t}{96 \sqrt{2} \max \left\{K_{1}, 2 K_{2}\right\}} 
	\geq \int_{\frac{\kappa(1-\kappa)t}{16 \max \left\{K_{1}, 2 K_{2}\right\}}}^{\sqrt{t}}  
	\sqrt{\log \mathcal{N}\left(u,\left\{g \in \mathcal{G}: \frac{1}{n} \sum_{i=1}^{n} g\left(\boldsymbol{x}_{i}\right)^{2} \leq 16 t\right\}, ||\cdot||_{1,n}\right)} d u.
	\end{align*}
	Then,
	\begin{align*}
	\mathbf{P}\left\{\sup _{g \in \mathcal{G}} \frac{\left|\mathbb{E}\{g(\boldsymbol{X})\}-\frac{1}{n} \sum_{i=1}^{n} g\left(\boldsymbol{X}_{i}\right)\right|}{\alpha+\mathbb{E}\{g(\boldsymbol{X})\}}>\kappa\right\} \nonumber 
	\leq 60 \exp \left(-\frac{n \alpha \kappa^{2}(1-\kappa)}{128 \cdot 2304 \max \left\{K_{1}^{2}, K_{2}\right\}}\right).
	\end{align*}
\end{lemma}

\subsection{Proof of Theorem \ref{theorem1}}\label{appA_3}

Let $\tau := \gamma - \frac{5}{2}$.
It is enough to show the main statement for $0 <\tau < 1$.
Note that $\varepsilon_{n} = n^{-\frac{\beta}{2\beta+d}} (\log n)^{\gamma}  
= n^{-\frac{\beta}{2\beta+d}} (\log n)^{\tau + \frac{5}{2}}$.
For $C_p$ defined in Lemma \ref{kohler2021rate}, we define $s_n$ as
\begin{align*}
s_n := \left\lceil C_p \left(n^\frac{d}{2\beta + d} (\log n)^{3\tau} \right)^{1/2} \right\rceil
\end{align*}
and $\bm{s}_n = (d, s_n, \dots, s_n, 1)^{\top} \in \mathbb{N}^{L_n + 2}$. 
We define
\begin{align*} 
    T_n := (d+1) p_n + (L_n -1) p_n (p_n + 1) + (p_n + 1) 
\end{align*}
and
\begin{align*} 
    S_n := (d+1) s_n + (L_n -1) s_n (s_n + 1) + (s_n + 1) 
\end{align*}
as the numbers of parameters in the DNNs with $(L_n , \bm{p}_n)$ and $(L_n , \bm{s}_n)$ architectures, respectively.
Let $\mathcal{F}_n$ be the set of pairs of truncated mDNN with $(L_n, \bm{p}_n)$ architecture and variances of the Gaussian noise,     
\begin{align*}
\mathcal{F}_n := \Big\{(f , \sigma^2)^{\top} \ : \ f \in \mathcal{F}^{\operatorname{mDNN}}(L_n, \bm{p}_n, F),\ 0 < \sigma^2 \leq \sigma_{\operatorname{max}}^2 \Big\}. 
\end{align*}
Also, we let $\mathcal{F}_n^{\prime} \subset \mathcal{F}_n$ by
\begin{align*}
\mathcal{F}_n^{\prime} := \Big\{(f , \sigma^2)^{\top} \ : \ f  \in \mathcal{F}^{\operatorname{mDNN}}(L_n, \bm{p}_n, F, s_n),\ 0 < \sigma^2 \leq \sigma_{\operatorname{max}}^2 \Big\}. 
\end{align*}

In the \hyperref[st1]{first step} of the proof, we fix $\{ \boldsymbol{x}^{(n)} \}_{n=1}^{\infty}$ and show 
\begin{equation} \label{empirical result_2}
\mathbb{E}_0 \left[ \Pi_n \left( (f, \sigma^2)^{\top} \in \mathcal{F}_n^{\prime} : || f - f_0 ||_{2, n} +
|\sigma^2 - \sigma_0^2 |> M_n \varepsilon_{n}  \middle\vert \mathcal{D}^{(n)}\right) \middle\vert \boldsymbol{X}^{(n)} = \boldsymbol{x}^{(n)}\right] \rightarrow 0
\end{equation}
as $n \to \infty$ for any $M_n \to \infty$. 

In the \hyperref[st2]{second step} of proof, we extend empirical $L_2$ error to expected $L_2$ error. 
In other words, we show
\begin{equation} \label{main result_2_1}
\mathbb{E}_0 \left[ \Pi_n \left( (f, \sigma^2)^{\top} \in \mathcal{F}_n^{\prime} : || f - f_0 ||_{2, \mathrm{P}_{X}}+
|\sigma^2 - \sigma_0^2 | > M_n \varepsilon_{n}  \middle\vert \mathcal{D}^{(n)}\right)\right] \rightarrow 0
\end{equation}
as $n \to \infty$ for any $M_n \to \infty$.

In the \hyperref[st3]{last step} of proof, we show 
\begin{equation} \label{main result_2_2}
\mathbb{E}_0 \left[ \Pi_n \left( (f, \sigma^2)^{\top} \in \left(\mathcal{F}_n \setminus \mathcal{F}_n^{\prime} \right) \middle\vert \mathcal{D}^{(n)}\right)\right] \rightarrow 0.
\end{equation}
and the proof of Theorem \ref{theorem1} is done by (\ref{main result_2_1}) and (\ref{main result_2_2}).

\textbf{Step 1}\label{st1}
For fixed $\{ \boldsymbol{x}^{(n)} \}_{n=1}^{\infty}$, let $P_{(f, \sigma^2), i}$ and $p_{(f, \sigma^2), i}$ be the probability measure and density corresponding to Gaussian distribution $N(f(\bm{x}_i), \sigma^2)$, respectively.
We define the semimetric $h_n^2$ on $\mathcal{F}_n^{\prime}$ as the average of the squares of the Hellinger distances for the distributions of the $n$ individual observations. In other words, for $(f_1, \sigma_1^2), (f_2, \sigma_2^2) \in \mathcal{F}_n^{\prime}$,
\begin{align*}
    h_{n}^{2}\left((f_1, \sigma_1^2), (f_2, \sigma_2^2)\right) := \frac{1}{n} \sum_{i=1}^{n} \int\left(\sqrt{p_{(f_1, \sigma_1^2), i}}-\sqrt{p_{(f_2, \sigma_2^2), i}}\right)^{2} 
    d P_{(f_1, \sigma_1^2), i}.
\end{align*}
Note that $h_n^2$ satisfies
\begin{align*} 
    (||f_1 - f_2||_{2,n} + |\sigma_1^2 - \sigma_2^2|)^2
    &\leq 2 \left( || f_1 - f_2 ||_{2, n}^2 + |\sigma_1^2 - \sigma_2^2 |^2 \right)  \\
     &\lesssim h_n^2 \left((f_1 , \sigma_1^2 ), (f_2 , \sigma_2^2 )\right).
\end{align*}
by Lemma B.1 of \citet{xie2020adaptive}.
Hence, to prove (\ref{empirical result_2}), it is suffices to show 
\begin{align}
\mathbb{E}_0 \left[ \Pi_n \left( (f, \sigma^2)^{\top} \in \mathcal{F}_n^{\prime} : 
h_n \left((f , \sigma^2 ), (f_0 , \sigma_0^2 )\right)> M_n \varepsilon_{n}  \middle\vert \mathcal{D}^{(n)}\right) \middle\vert \boldsymbol{X}^{(n)} = \boldsymbol{x}^{(n)}\right] \rightarrow 0. \label{emp_h}    
\end{align}
Since Hellinger distance possesses an exponentially powerful local test with respect to both the type-I and type-II errors (Lemma 2 of \citet{ghosal2007convergence}), we can use the standard tool to establish concentration rates that we restate in Lemma \ref{ghosal2007convergence} for the convenience of the reader.

If we define semimetric $d_n$ on $\mathcal{F}_n^{\prime}$ as 
\begin{align*}
d_n^2 \left((f_1 , \sigma_1^2 ), (f_2 , \sigma_2^2 )\right) := ||f_1 - f_2||_{1,n} + |\sigma_1^2 - \sigma_2^2|^2,	
\end{align*}
then $h_n^2(\cdot) \lesssim d_n^2(\cdot)$ holds by by Lemma B.1 of \citet{xie2020adaptive}
and hence $\mathcal{N}\left(\varepsilon ,\mathcal{F}_{n}^{\prime}, h_{n}\right) 
\leq  \mathcal{N}\left(\varepsilon^2 ,\mathcal{F}_{n}^{\prime}, d_n^2 \right)$.
Also, by the fact that $||f_1 - f_2||_{1,n} \leq \frac{\varepsilon^2}{2}$ and
$|\sigma_1^2 - \sigma_2^2|^2 \leq \frac{\varepsilon^2}{2}$ implies
$||f_1 - f_2||_{1,n} + |\sigma_1^2 - \sigma_2^2|^2 \leq  \varepsilon^2$
and for every $f_{\bm{M}, \bm{\theta} \ [-F,F]}^{\operatorname{mDNN}} \in  \mathcal{F}^{\operatorname{mDNN}}(L_n, \bm{p}_n, F, s_n)$ there exist $\bm{\psi}_{\bm{M}, \bm{\theta}} \in \mathbb{R}^{S_n}$ and $f_{\bm{\psi}_{\bm{M}, \bm{\theta}} \ [-F,F]}^{\operatorname{DNN}} \in  \mathcal{F}^{\operatorname{DNN}}(L_n, \bm{s}_n, F)$ such that 
$f_{\bm{M}, \bm{\theta} \ [-F,F]}^{\operatorname{mDNN}} 
= f_{\bm{\psi}_{\bm{M}, \bm{\theta}} \ [-F,F]}^{\operatorname{DNN}}$ holds, we get
\begin{align*}
\mathcal{N}\left(\varepsilon ,\mathcal{F}_{n}^{\prime}, h_{n}\right) 
\leq & \mathcal{N}\left(\varepsilon^2 ,\mathcal{F}_{n}^{\prime}, d_n^2 \right)\\
\leq &\mathcal{N}\left(\frac{\varepsilon^2}{2} , 
\mathcal{F}^{\operatorname{mDNN}}(L_n, \bm{p}_n, F, s_n), ||\cdot||_{1,n}\right)
\frac{\sqrt{2}\sigma_{max}^2}{\varepsilon}\\
\leq &\mathcal{N}\left(\frac{\varepsilon^2}{2} , 
\mathcal{F}^{\operatorname{DNN}}(L_n, \bm{s}_n, F), ||\cdot||_{1,n}\right)
\frac{\sqrt{2}\sigma_{max}^2}{\varepsilon}.
\end{align*}
Since functions in $\mathcal{F}^{DNN}(L_n, \bm{s}_n, F)$ are bounded by $[-F, F]$, there exists $c_1>0$ such that
\begin{align*}
\mathcal{N}\left(\frac{\varepsilon^2}{2} , 
\mathcal{F}^{\operatorname{DNN}}(L_n, \bm{s}_n, F), ||\cdot||_{1,n}\right) \leq& \mathcal{M}\left(\frac{\varepsilon^2}{2} ,\mathcal{F}^{DNN}(L_n, \bm{s}_n, F), ||\cdot||_{1,n}\right) \\
\leq& 3\left(\frac{8 e F}{\epsilon^{2}} \log \frac{12 e F}{\epsilon^{2}}\right)^{V_{\mathcal{F}^{\operatorname{DNN}}(L_n, \bm{s}_n,F)}^{+}}\\
\leq& 3\left(\frac{8 e F}{\epsilon^{2}} \log \frac{12 e F}{\epsilon^{2}}\right)
^{c_1 L_n S_n \log S_n}		
\end{align*}
holds for every $\varepsilon > 0$ by Theorem 9.4 of \citet{gyorfi2002distribution} and Theorem 6 of \citet{harvey2017nearly}.
Hence,
\begin{align} 
\sup _{\varepsilon>\varepsilon_{n}} \log \mathcal{N} \Big(\varepsilon, \mathcal{F}_{n}^{\prime}, h_{n}\Big)
\lesssim& L_n S_n \log S_n \log n \nonumber\\
\asymp& n^{\frac{d}{2\beta + d}} (\log n)^{4+3\tau} \nonumber  \\
\leq & n \varepsilon_{n}^2 \label{entropyupp} 
\end{align}
holds.

Now, we define
\begin{align*}
K_i ((f_0, \sigma_0^2), (f, \sigma^2)) =& \int \log (p_{(f_0, \sigma_0^2), i} / p_{(f, \sigma^2), i}) dP_{(f_0, \sigma_0^2), i},\\
V_{2,0;i} ((f_0, \sigma_0^2), (f, \sigma^2)) =& \int  \left( \log (p_{(f_0, \sigma_0^2), i} / p_{(f, \sigma^2), i} - K_i ((f_0, \sigma_0^2), (f, \sigma^2))\right)^2 dP_{(f_0, \sigma_0^2), i}
\end{align*}
and
\begin{align*}
B_{n}^{*}\left((f_0, \sigma_0^2), \varepsilon ; 2\right)=\big\{(f, \sigma^2) \in \mathcal{F}_n^{\prime}:  \frac{1}{n}\sum_{i=1}^n K_i \left((f_0, \sigma_0^2), (f, \sigma^2)\right) \leq \varepsilon^{2}, \\
\frac{1}{n}\sum_{i=1}^n V_{2, 0 ; i}\left((f_0, \sigma_0^2), (f, \sigma^2)\right) \leq \varepsilon^{2}\big\}.
\end{align*}
with notations on Lemma \ref{ghosal2007convergence}.
In addition, for $\varepsilon>0$, define
\begin{align*}
A_{n}^{*}\left((f_0 , \sigma_0^2), \varepsilon ; 2\right) := \Big\{ (f , \sigma^2) \in \mathcal{F}_n^{\prime} : \underset{i}{\max} |f(\bm{x}_i) - f_0(\bm{x}_i)| \leq \frac{\sigma_0 \varepsilon}{2}, \nonumber \\
\sigma^2 \in [\sigma_0^2 , (1+\varepsilon^2)\sigma_0^2] \Big\}.
\end{align*}
Then for every $f  \in A_{n}^{*}\left((f_0, \sigma_0^2), \varepsilon ; 2\right)$ and $i \in [n],$
\begin{align*}
K_i((f_0, \sigma_0^2), (f, \sigma^2)) = \frac{1}{2} \log \frac{\sigma^2}{\sigma_0^2} + \frac{\sigma_0^2 + (f_0(\bm{x}_i) - f(\bm{x}_i))^2}{2\sigma^2} - \frac{1}{2}
\leq \varepsilon^2
\end{align*} 
and
\begin{align*}
V_{2,0;i} ((f_0, \sigma_0^2), (f, \sigma^2)) =& \operatorname{Var}_{f_0, \sigma_0^2}\left(-\frac{(Y_i - f_0(\bm{x}_i))^2}{2\sigma_0^2}
+ \frac{(Y_i - f(\bm{x}_i))^2}{2\sigma^2}\right)\\
=& \operatorname{Var}_{f_0, \sigma_0^2}\left( -\frac{1}{2}(1-\frac{\sigma_0^2}{\sigma^2})Z_i^2 + \frac{\sigma_0 (f_0(\bm{x}_i) - f(\bm{x}_i)) Z_i}{\sigma^2}\right)
\leq \varepsilon^2
\end{align*}
where $Y_i \sim N(f_0(\bm{x}_i), \sigma_0^2)$ and $Z_i := \frac{Y_i - f_0(\bm{x}_i)}{\sigma_0} \sim N(0,1)$.
Hence, we can conclude that 
\begin{align}
    A_{n}^{*}\left((f_0, \sigma_0^2), \varepsilon_n ; 2\right) \subset B_{n}^{*}\left((f_0, \sigma_0^2), \varepsilon_n ; 2\right) \label{A_in_B}.
\end{align}

Now we define $v_n$ as
\begin{align*}
v_n := \left\lceil C_p \left(n^\frac{d}{2\beta + d} (\log n)^{-1} \right)^{1/2} \right\rceil,
\end{align*}
which is the width of the network in Lemma \ref{kohler2021rate}.
Let
$\bm{v}_n = (d, v_n, \dots, v_n, 1)^{\top} \in \mathbb{N}^{L_n + 2}$, and let
\begin{align*} 
    V_n := (d+1) v_n + (L_n -1) v_n (v_n + 1) + (v_n + 1) \label{V_n},
\end{align*}
which is the number of parameters in DNN with $(L_n , \bm{v}_n)$ architecture.
For any $\bm{\theta} \in \mathbb{R}^{T_n}$ and any $\bm{M}$ with $|\bm{m}^{(l)}|_0 = v_n$ for $l \in [L]$,
there exists $V_n$ dimension index vector $\bm{j}_{\bm{M}} \subset [T_n]$ such that 
$f_{ \bm{\psi}_{\bm{M}, \bm{\theta}}}^{\operatorname{DNN}} = f_{\bm{M}, \bm{\theta}}^{\operatorname{mDNN}}$
holds for $\bm{\psi}_{\bm{M}, \bm{\theta}} := (\bm{\theta})_{\bm{j}_{\bm{M}}} \in \mathbb{R}^{V_n}$.
In other words, $f_{ \bm{\psi}_{\bm{M}, \bm{\theta}}}^{\operatorname{DNN}}$ is the sub-network of
$f_{\bm{M}, \bm{\theta}}^{\operatorname{DNN}}$ consisting of the unmasked nodes.
Also, there exists $\hat{\bm{\psi}} \in [-n, n]^{V_n}$ such that
\begin{equation}
\left\|f_{\hat{\bm{\psi}} \ [-F,F] }^{\operatorname{DNN}} - f_{0}\right\|_{\infty} 
< \frac{\sigma_0 \varepsilon_n}{4} \label{thetahat_def}
\end{equation}
satisfies for large n by Lemma \ref{kohler2021rate}. 

With (\ref{A_in_B}), (\ref{thetahat_def}) and Lemma \ref{lemma_similar}, we can obtain the lower bound of $\Pi_n\left(B_{n}^{*}\left((f_0 , \sigma_0^2), \varepsilon_n ; 2\right)\right)$ by
\begin{align}
&\Pi_n\left( B_{n}^{*}\left((f_0 , \sigma_0^2), \varepsilon_n ; 2\right) \right) \nonumber\\
\geq &\Pi_n\left( A_{n}^{*}\left((f_0 , \sigma_0^2), \varepsilon_n ; 2\right) \right) \nonumber\\
= &\Pi_n \Big( \left\{ (\bm{M}, \bm{\theta})  : 
\underset{i}{\max}\ |f_{\bm{M}, \bm{\theta} \ [-F,F]}^{\operatorname{mDNN}}(\bm{x}_i) - f_0(\bm{x}_i)| \leq \frac{\sigma_0 \varepsilon_n}{2} \right\} \Big)
\Pi_n \Big( \sigma^2 \in [\sigma_0^2 , (1+\varepsilon_n^2)\sigma_0^2] \Big)\nonumber\\
\geq & \Pi_n \left( \left\{ (\bm{M}, \bm{\theta}) : |\bm{m}^{(1)}|_0=\cdots = |\bm{m}^{(L)}|_0 = v_n ,\  \underset{i}{\max}\ \left|f_{ \bm{\psi}_{\bm{M}, \bm{\theta}} \ [-F,F]}^{\operatorname{DNN}}(\bm{x}_i) - f_0(\bm{x}_i)\right| \leq \frac{\sigma_0 \varepsilon_n}{2} \right\} \right)
\nonumber\\
& \qquad \times \Pi_n \Big( \sigma^2 \in [\sigma_0^2 , (1+\varepsilon_n^2)\sigma_0^2] \Big) \nonumber\\
\geq & \Pi_n \left( \left\{ (\bm{M}, \bm{\theta}) : |\bm{m}^{(1)}|_0=\cdots = |\bm{m}^{(L)}|_0 = v_n ,\  \underset{i}{\max}\ \left|f_{ \bm{\psi}_{\bm{M}, \bm{\theta}} \ [-F,F]}^{\operatorname{DNN}}(\bm{x}_i) - f_{\hat{\bm{\psi}} \ [-F,F] }^{\operatorname{DNN}}(\bm{x}_i)\right| \leq \frac{\sigma_0 \varepsilon_n}{4} \right\} \right)
\nonumber\\
& \qquad \times \Pi_n \Big( \sigma^2 \in [\sigma_0^2 , (1+\varepsilon_n^2)\sigma_0^2] \Big) \nonumber\\
\geq & \Pi_n \left( \left\{ (\bm{M}, \bm{\theta}) : |\bm{m}^{(1)}|_0=\cdots = |\bm{m}^{(L)}|_0 = v_n ,\   \left|\bm{\psi}_{\bm{M}, \bm{\theta}}-\hat{\bm{\psi}}\right|_{\infty} 
\leq \frac{\sigma_0 \varepsilon_n}{4d v_n^{L_n} n^{L_n + 1} (L_n + 1) } \right\} \right)
\nonumber\\
& \qquad \times \Pi_n \Big( \sigma^2 \in [\sigma_0^2 , (1+\varepsilon_n^2)\sigma_0^2] \Big) \nonumber\\
\gtrsim& \exp \left(-(\lambda \log n)^5 v_n^2 L_n \right) 
\exp\left(-V_n (\log n)^{2}\right) \varepsilon_{n}^2  \nonumber\\
\gtrsim& \exp \left(-\lambda^5 C_p^2 C_L (\log n)^5 n^{\frac{d}{2\beta + d}}\right)
\exp \left(- C_p^2 C_L (\log n)^{2} n^{\frac{d}{2\beta+d} }\right) n^{-1}. \label{blower}
\end{align}
Hence we have that 
\begin{align} \label{lower}
\Pi_n\left( B_{n}^{*}\left(\eta_{0}, \varepsilon_n ; 2\right) \right) \geq e^{-n \varepsilon_{n}^2} 
\end{align}
for all but finite many $n$.
Hence by (\ref{entropyupp}), (\ref{lower}) and Lemma \ref{ghosal2007convergence}, the proof of (\ref{emp_h}) is done. $\quad \square$

\textbf{Step 2.}\label{st2}
Since (\ref{empirical result_2}) holds for arbitrary $\{ \boldsymbol{x}^{(n)} \}_{n=1}^{\infty}$,
\begin{equation} \label{empirical result_2_2}
\mathbb{E}_0 \left[ \Pi_n \left( (f, \sigma^2)^{\top} \in \mathcal{F}_n^{\prime} : || f - f_0 ||_{2, n} +
|\sigma^2 - \sigma_0^2 |> M_n \varepsilon_{n}  \middle\vert \mathcal{D}^{(n)}\right) \right] \rightarrow 0
\end{equation}
also holds.
Next, we will check the conditions in Lemma \ref{gyorfi2002distribution} for
\begin{align*}
\mathcal{G} :=& \left\{ g \ : \ g=(f_{\bm{M}, \bm{\theta} \ [-F,F]}^{\operatorname{mDNN}} - f_0)^2 , f_{\bm{M}, \bm{\theta} \ [-F,F]}^{\operatorname{mDNN}} \in \mathcal{F}^{\operatorname{mDNN}}(L_n, \bm{p}_n, F, s_n) \right\},\\
\kappa :=& \frac{1}{2} ,\ \alpha := \varepsilon_{n}^2 ,\ K_1 = K_2 = 4F^2 .
\end{align*}
First, it is easy to check $||g(\boldsymbol{x})||_{\infty} \leq 4F^2$ and $\mathbb{E}(g(\boldsymbol{X})^2) \leq 4F^2 \mathbb{E}(g(\boldsymbol{X}))$ for $g \in \mathcal{G}$.
Also, for every $f_{\bm{M}, \bm{\theta} \ [-F,F]}^{\operatorname{mDNN}} \in  \mathcal{F}^{\operatorname{mDNN}}(L_n, \bm{p}_n, F, s_n)$, 
there exist $\bm{\psi}_{\bm{M}, \bm{\theta}} \in \mathbb{R}^{S_n}$ and $f_{\bm{\psi}_{\bm{M}, \bm{\theta}} \ [-F,F]}^{\operatorname{DNN}} \in  \mathcal{F}^{\operatorname{DNN}}(L_n, \bm{s}_n, F)$ such that 
$f_{\bm{M}, \bm{\theta} \ [-F,F]}^{\operatorname{mDNN}} 
= f_{\bm{\psi}_{\bm{M}, \bm{\theta}} \ [-F,F]}^{\operatorname{DNN}}$ holds.
Since
\begin{equation*}
\left\| (f_{\bm{\psi}_1 \ [-F,F]}^{\operatorname{DNN}} - f_0)^2 
- (f_{\bm{\psi}_2 \ [-F,F]}^{\operatorname{DNN}} - f_0)^2 \right\|_{n, 1}  
\leq 4F \left\| f_{\bm{\psi}_1 \ [-F,F]}^{\operatorname{DNN}}  - f_{\bm{\psi}_2 \ [-F,F]}^{\operatorname{DNN}}\right\|_{n, 1} \label{step2_ineq1}
\end{equation*}
holds for $\bm{\psi}_1, \bm{\psi}_2 \in \mathbb{R}^{S_n}$, there exists $c_2>0$ such that
\begin{align*}
\mathcal{N} \left( u, \mathcal{G}, ||\cdot||_{n, 1} \right)
\leq& \mathcal{N}\left( \frac{u}{4F}, \mathcal{F}^{\operatorname{DNN}}(L_n, \bm{s}_n, F), ||\cdot||_{n, 1} \right) \\
\leq& \mathcal{M}\left( \frac{u}{4F}, \mathcal{F}^{\operatorname{DNN}}(L_n, \bm{s}_n, F), ||\cdot||_{n, 1} \right) \\
\leq& 3\left(\frac{16 e F^2}{u} \log \frac{24 e F^2}{u}\right)^{V_{\mathcal{F}^{\operatorname{DNN}}(L_n, \bm{s}_n, F)}^{+}}\\
\lesssim& n^{c_2 S_n L_n \log S_n}
\end{align*}
for $u \geq n^{-1}$ by Theorem 9.4 of \citet{gyorfi2002distribution} and Theorem 6 of \citet{harvey2017nearly}.
Hence for all $t \geq \frac{\varepsilon_n^2}{8}$,
\begin{align*}
\int_{\frac{\kappa(1-\kappa)t}{16 \max \left\{K_{1}, 2 K_{2}\right\}}}^{\sqrt{t}}  
\sqrt{\log \mathcal{N} \left( u, \mathcal{G}, ||\cdot||_{n, 1} \right)} d u 
\lesssim& \sqrt{t} \left( n^{\frac{d}{2\beta + d}} (\log n)^{4+3\tau} \right)^{\frac{1}{2}}\\
=& o\left( \frac{\sqrt{n} t/4}{96 \sqrt{2} \max \left\{K_1, 2K_2 \right\}} \right)
\end{align*}
holds.
To sum up, we conclude that
\begin{align}
\mathbf{P}\left\{\sup _{f \in \mathcal{F}^{\operatorname{DNN}}(L_n, \bm{p}_n^{\prime})} \frac{\left| ||f-f_0||_{2, \mathrm{P}_{X}}^2  - ||f-f_0||_{2, n}^2 \right|}{\varepsilon_{n}^2+||f-f_0||_{2, \mathrm{P}_{X}}^2}>\frac{1}{2}\right\}
\leq 60 \exp \left(-\frac{n \varepsilon_{n}^2 / 8}{128 \cdot 2304 \cdot 16F^4}\right)
\label{last}
\end{align}
holds for all but finite many $n$ by Lemma \ref{gyorfi2002distribution}. 
Hence by (\ref{empirical result_2_2}) and (\ref{last}), the proof of (\ref{main result_2_1}) is done. $\quad \square$

\textbf{Step 3.}\label{st3}
Since
\begin{align*}
    \left( \frac{1}{2ks} - \frac{1}{4 k^2 s^3} \right)e^{-ks^2} 
    \leq \int_{s}^{\infty} e^{-kt^2} dt
    \leq \frac{1}{2ks} e^{-ks^2}
\end{align*}
for any $k>0$ and $s>0$,
\begin{align*}
\Pi_n(|\bm{m}^{(l)}|_0 > s_n) 
\leq & \frac{\sum_{s = s_n + 1}^{p_n} e^{- (\lambda \log n)^5 {s}^2}}{e^{-(\lambda \log n)^5}}\\
\lesssim& e^{- (\lambda \log n)^5 {s_n}^2} e^{(\lambda \log n)^5} (\log n)^{-5} 
\end{align*}
holds for every $l \in [L]$. Then we have that
\begin{align}
\Pi_{n}\left(\mathcal{F}_n \setminus \mathcal{F}_{n}^{\prime}\right) 
=& \Pi_{n}\left( \bigcup_{l \in [L]} \{|\bm{m}^{(l)}|_0 > s_n\}\right) \nonumber\\
\leq& L_n \cdot \Pi_{n}\left( \{|\bm{m}^{(1)}|_0 > s_n\}\right) \nonumber\\
\lesssim& \exp\left(-\lambda^5 C_p^2 (\log n)^{5+3\tau} n^{\frac{d}{2\beta + d} }\right)
\exp \left( (\lambda \log n)^5 \right) (\log n)^{-4}. \label{last2}	
\end{align}
By (\ref{blower}) and (\ref{last2}), we obtain
\begin{align*}
&\frac{\Pi_{n}\left(\mathcal{F}_n \setminus \mathcal{F}_{n}^{\prime}\right)}
{\Pi_{n}\left(B_{n}^{\star}\left(\eta_{0}, \varepsilon_{n} ; 2\right)\right)e^{-2n\varepsilon_{n}^2}}\\ 
& \quad \lesssim \frac{\exp\left(-\lambda^5 C_p^2 (\log n)^{5+3\tau} n^{\frac{d}{2\beta + d} }\right) \exp \left( (\lambda \log n)^5 \right)}
{\exp(-\lambda^5 C_p^2 C_L (\log n)^5 n^{\frac{d}{2\beta + d}}) 
	\exp(-C_p^2 C_L (\log n)^2 n^{\frac{d}{2\beta + d}})n^{-1} 
	\exp(-2 (\log n)^{5+2\tau} n^{\frac{d}{2\beta+d}})}\\
& \quad =o(1),
\end{align*}
and thus we have
\begin{equation*}
\mathbb{E}_0 \left[ \Pi_n \left( (f, \sigma^2)^{\top} \in \left(\mathcal{F}_n \setminus \mathcal{F}_n^{\prime} \right) \middle\vert \mathcal{D}^{(n)}\right)\right] \rightarrow 0.
\end{equation*}
by apply Lemma 1 of \citet{ghosal2007convergence}, which completes the proof of (\ref{main result_2_2}).
$\quad \square$

\subsection{Proof of Theorem \ref{theorem2}}\label{appA_4}
The proof is a slight modification of the proof of Theorem \ref{theorem1}.
Let $\tau := \gamma - \frac{5}{2}$. 
It suffices to show the main statement for $0 <\tau < 1$.
Note that $\varepsilon_{n} = n^{-\frac{\beta}{2\beta+d}} (\log n)^{\gamma}  
= n^{-\frac{\beta}{2\beta+d}} (\log n)^{\tau + \frac{5}{2}}$.
For $C_p$ defined in Lemma \ref{kohler2021rate}, we define $s_n$ as
\begin{align*}
s_n := \left\lceil C_p \left(n^\frac{d}{2\beta + d} (\log n)^{3\tau} \right)^{1/2} \right\rceil
\end{align*}
and $\bm{s}_n = (d, s_n, \dots, s_n, 1)^{\top} \in \mathbb{N}^{L_n + 2}$. 
We define
\begin{align*} 
    T_n := (d+1) p_n + (L_n -1) p_n (p_n + 1) + (p_n + 1) 
\end{align*}
and
\begin{align*} 
    S_n := (d+1) s_n + (L_n -1) s_n (s_n + 1) + (s_n + 1) 
\end{align*}
as the numbers of parameters in the DNNs with $(L_n , \bm{p}_n)$ and $(L_n , \bm{s}_n)$ architectures, respectively.
Let $\mathcal{F}_n$ be the set of truncated mDNN with the $(L_n, \bm{p}_n)$ architecture,     
\begin{align*}
\mathcal{F}_n := \mathcal{F}^{\operatorname{mDNN}}(L_n, \bm{p}_n, F).
\end{align*}
Also, we let $\mathcal{F}_n^{\prime} \subset \mathcal{F}_n$ by
\begin{align*}
\mathcal{F}_n^{\prime} := \mathcal{F}^{\operatorname{mDNN}}(L_n, \bm{p}_n, F, s_n)
\end{align*}

In the \hyperref[st1-cl]{first step} of proof, we fix $\{ \boldsymbol{x}^{(n)} \}_{n=1}^{\infty}$ and show 
\begin{equation} \label{empirical result_2_cl}
\mathbb{E}_0 \left[ \Pi_n \left( f \in \mathcal{F}_n^{\prime} : || \phi \circ f - \phi \circ f_0 ||_{2, n} 
> M_n \varepsilon_{n}  \middle\vert \mathcal{D}^{(n)}\right) \middle\vert \boldsymbol{X}^{(n)} = \boldsymbol{x}^{(n)}\right] \rightarrow 0
\end{equation}
as $n \to \infty$ for any $M_n \to \infty$. 

In the \hyperref[st2]{second step} of proof, we extend empirical $L_2$ error to expected $L_2$ error. 
In other words, we show
\begin{equation} \label{main result_2_1_cl}
\mathbb{E}_0 \left[ \Pi_n \left( f \in \mathcal{F}_n^{\prime} : || \phi \circ f - \phi \circ f_0 ||_{2, \mathrm{P}_{X}} > M_n \varepsilon_{n}  \middle\vert \mathcal{D}^{(n)}\right)\right] \rightarrow 0
\end{equation}
as $n \to \infty$ for any $M_n \to \infty$.

Then, since we already showed 
\begin{equation} \label{main result_2_2_cl}
\mathbb{E}_0 \left[ \Pi_n \left( f \in \left(\mathcal{F}_n \setminus \mathcal{F}_n^{\prime} \right) \middle\vert \mathcal{D}^{(n)}\right)\right] \rightarrow 0
\end{equation}
in the \hyperref[st3]{last step} of the proof of Theorem \ref{theorem1}, the proof of Theorem \ref{theorem2} is done by (\ref{main result_2_1_cl}) and (\ref{main result_2_2_cl}).

\textbf{Step 1}\label{st1-cl}
For fixed $\{ \boldsymbol{x}^{(n)} \}_{n=1}^{\infty}$, let $P_{f, i}$ and $p_{f, i}$ be the probability measure and density corresponding to the Bernoulli distribution $\operatorname{Ber}(\phi \circ f(\bm{x}_i))$, respectively.
We define the semimetric $h_n^2$ on $\mathcal{F}_n^{\prime}$ as the average of the squares of the Hellinger distances for the distributions of the $n$ individual observations. In other words, for $f_1 , f_2 \in \mathcal{F}_n^{\prime}$,
\begin{align*}
    h_{n}^{2}\left(f_1, f_2\right) := \frac{1}{n} \sum_{i=1}^{n} \int\left(\sqrt{p_{f_1 , i }}-\sqrt{p_{f_2 , i}}\right)^{2} 
    d P_{f_1 , i}.
\end{align*}
Also, we define semimetric $d_n$ on $\mathcal{F}_n^{\prime}$ as 
\begin{align*}
d_n \left(f_1 , f_2 \right) := ||\phi \circ f_1 - \phi \circ f_2||_{2,n}.	
\end{align*}

Note that since $f_1 , f_2 \in \mathcal{F}_n^{\prime}$ are bounded,
\begin{align*}
    d_n^2 \left(f_1 , f_2 \right) 
    =& \frac{1}{n} \sum_{i=1}^{n} (\phi \circ f_1(\bm{x}_i)-\phi \circ f_2(\bm{x}_i))^2 \\
    = &  \frac{1}{n} \sum_{i=1}^{n} \left( \sqrt{\phi \circ f_1(\bm{x}_i)}-\sqrt{\phi \circ f_2(\bm{x}_i)}\right)^2 
    \left(\sqrt{\phi \circ f_1(\bm{x}_i)}+\sqrt{\phi \circ f_2(\bm{x}_i)}\right)^2  \\
    \lesssim & h_{n}^{2}\left(f_1, f_2\right)  
\end{align*}
and
\begin{align*}
    d_{n}^{2}\left(f_1, f_2\right) 
    = &  \frac{1}{2n} \sum_{i=1}^{n} \Bigg\{ \left( \sqrt{\phi \circ f_1(\bm{x}_i)}-\sqrt{\phi \circ f_2(\bm{x}_i)}\right)^2 
    \left(\sqrt{\phi \circ f_1(\bm{x}_i)}+\sqrt{\phi \circ f_2(\bm{x}_i)}\right)^2  \\
    & + \left( \sqrt{1-\phi \circ f_1(\bm{x}_i)}-\sqrt{1-\phi \circ f_2(\bm{x}_i)}\right)^2 
    \left(\sqrt{1-\phi \circ f_1(\bm{x}_i)}+\sqrt{1-\phi \circ f_2(\bm{x}_i)}\right)^2\Bigg\}\\
    & \gtrsim h_{n}^{2}\left(f_1, f_2\right)
\end{align*}
holds.
Hence, to prove (\ref{empirical result_2_cl}), it suffices to show 
\begin{align}
\mathbb{E}_0 \left[ \Pi_n \left( (f, \sigma^2)^{\top} \in \mathcal{F}_n^{\prime} : 
h_n \left(f, f_0\right)> M_n \varepsilon_{n}  \middle\vert \mathcal{D}^{(n)}\right) \middle\vert \boldsymbol{X}^{(n)} = \boldsymbol{x}^{(n)}\right] \rightarrow 0. \label{emp_h_cl}    
\end{align}

Since Hellinger distance possesses an exponentially powerful local test with respect to both the type-I and type-II errors (Lemma 2 of \citet{ghosal2007convergence}), we can use standard tools to establish concentration rates that we restate in Lemma \ref{ghosal2007convergence} for the convenience of the reader.

Since $\phi$ is L1-Lipschitz function, 
\begin{align}
\mathcal{N}\left(\varepsilon ,\mathcal{F}_{n}^{\prime}, h_{n}\right) 
\lesssim &  \mathcal{N}\left(\varepsilon ,\mathcal{F}_{n}^{\prime}, d_n \right) \nonumber\\
\leq &  \mathcal{N}
\left(\varepsilon ,\mathcal{F}_{n}^{\prime}, ||\cdot||_{2,n} \right).  \nonumber
\end{align}
For every $f_{\bm{M}, \bm{\theta} \ [-F,F]}^{\operatorname{mDNN}} \in  \mathcal{F}^{\operatorname{mDNN}}(L_n, \bm{p}_n, F, s_n)$ there exist $\bm{\psi}_{\bm{M}, \bm{\theta}} \in \mathbb{R}^{S_n}$ and $f_{\bm{\psi}_{\bm{M}, \bm{\theta}} \ [-F,F]}^{\operatorname{DNN}} \in  \mathcal{F}^{\operatorname{DNN}}(L_n, \bm{s}_n, F)$ such that 
$f_{\bm{M}, \bm{\theta} \ [-F,F]}^{\operatorname{mDNN}} 
= f_{\bm{\psi}_{\bm{M}, \bm{\theta}} \ [-F,F]}^{\operatorname{DNN}}$ holds. So we get
\begin{align}
\mathcal{N}
\left(\varepsilon ,\mathcal{F}_{n}^{\prime}, ||\cdot||_{2,n} \right) 
= &
\mathcal{N}
\left(\varepsilon ,\mathcal{F}^{\operatorname{DNN}}(L_n, \bm{s}_n, F), ||\cdot||_{2,n} \right) \nonumber\\
\leq & \mathcal{M}
\left(\varepsilon ,\mathcal{F}^{\operatorname{DNN}}(L_n, \bm{s}_n, F), ||\cdot||_{2,n} \right). \nonumber
\end{align}
Since functions in $\mathcal{F}^{DNN}(L_n, \bm{s}_n, F)$ are bounded by $[-F, F]$, there exists $c_3>0$ such that
\begin{align*}
    \mathcal{M}
\left(\varepsilon ,\mathcal{F}^{\operatorname{DNN}}(L_n, \bm{s}_n, F), ||\cdot||_{2,n} \right)
\leq& 3\left(\frac{8 e F^2}{\epsilon^{2}} \log \frac{12 e F^2}{\epsilon^{2}}\right)^{V_{\mathcal{F}^{\operatorname{DNN}}(L_n, \bm{s}_n,F)}^{+}}\\
\leq& 3\left(\frac{8 e F^2}{\epsilon^{2}} \log \frac{12 e F^2}{\epsilon^{2}}\right)
^{c_3 L_n S_n \log S_n}	
\end{align*}
holds for every $\varepsilon > 0$ by Theorem 9.4 of \citet{gyorfi2002distribution} and Theorem 6 of \citet{harvey2017nearly}.
To sum up,
\begin{align} 
\sup _{\varepsilon>\varepsilon_{n}} \log \mathcal{N} \Big(\varepsilon, \mathcal{F}_{n}^{\prime}, h_{n}\Big)
\lesssim& L_n S_n \log S_n \log n \nonumber\\
\asymp& n^{\frac{d}{2\beta + d}} (\log n)^{4+3\tau} \nonumber  \\
\leq & n \varepsilon_{n}^2 \label{entropyupp_cl} 
\end{align}
holds.

Now, we define
\begin{align*}
K_i (f_0, f) =& \int \log (p_{f_0 , i} / p_{f , i}) dP_{f_0 , i},\\
V_{2,0;i} (f_0, f) =& \int  \left( \log (p_{f_0 , i} / p_{f , i}) - K_i (f_0, f)\right)^2 dP_{f_0 , i}, 
\end{align*}
and
\begin{align*}
B_{n}^{*}\left(f_0, \varepsilon ; 2\right)=\big\{f \in \mathcal{F}_n^{\prime}:  \frac{1}{n}\sum_{i=1}^n K_i \left(f_0, f\right) \leq \varepsilon^{2}, \\
\frac{1}{n}\sum_{i=1}^n V_{2, 0 ; i}\left(f_0, f\right) \leq \varepsilon^{2}\big\}.
\end{align*}
For $\varepsilon>0$, define
\begin{align*}
A_{n}^{*}\left(f_0 , \varepsilon ; 2\right) := \Big\{ f \in \mathcal{F}_n^{\prime} : \underset{i}{\max} |f(\bm{x}_i) - f_0(\bm{x}_i)| \leq  \varepsilon \Big\}.
\end{align*}
Then by Lemma 3.2 of \citet{van2008rates}, we can get 
\begin{align*}
    A_{n}^{*}\left(f_{0}, \varepsilon_n ; 2\right) \subset B_{n}^{*}\left(f_{0}, \varepsilon_n ; 2\right).
\end{align*}
Now by following the proof of (\ref{blower}), we obtain 
\begin{align} \label{lower_cl}
\Pi_n\left( B_{n}^{*}\left(\eta_{0}, \varepsilon_n ; 2\right) \right) \geq e^{-n \varepsilon_{n}^2} 
\end{align}
for all but finite many $n$.
Hence by (\ref{entropyupp_cl}), (\ref{lower_cl}) and Lemma \ref{ghosal2007convergence}, the proof of (\ref{emp_h_cl}) is done. $\quad \square$

\textbf{Step 2.}\label{st2_cl}
Since (\ref{empirical result_2_cl}) holds for arbitrary $\{ \boldsymbol{x}^{(n)} \}_{n=1}^{\infty}$,
\begin{equation} \label{empirical result_2_2_cl}
\mathbb{E}_0 \left[ \Pi_n \left( f \in \mathcal{F}_n^{\prime} : || \phi \circ f - \phi \circ f_0 ||_{2, n} > M_n \varepsilon_{n}  \middle\vert \mathcal{D}^{(n)}\right) \right] \rightarrow 0
\end{equation}
also holds.
Next, we will check the conditions in Lemma \ref{gyorfi2002distribution} for
\begin{align*}
\mathcal{G} :=& \left\{ g \ : \ g=(\phi \circ f_{\bm{M}, \bm{\theta} \ [-F,F]}^{\operatorname{mDNN}} - \phi \circ f_0)^2 , f_{\bm{M}, \bm{\theta} \ [-F,F]}^{\operatorname{mDNN}} \in \mathcal{F}^{\operatorname{mDNN}}(L_n, \bm{p}_n, F, s_n) \right\}\\
\kappa :=& \frac{1}{2} ,\ \alpha := \varepsilon_{n}^2 ,\ K_1 = K_2 = 1 .
\end{align*}
First, it is easy to check $||g(\boldsymbol{x})||_{\infty} \leq 1$ and $E(g(\boldsymbol{X})^2) \leq E(g(\boldsymbol{X}))$ for $g \in \mathcal{G}$.
Also, for every $f_{\bm{M}, \bm{\theta} \ [-F,F]}^{\operatorname{mDNN}} \in  \mathcal{F}^{\operatorname{mDNN}}(L_n, \bm{p}_n, F, s_n)$, 
there exist $\bm{\psi}_{\bm{M}, \bm{\theta}} \in \mathbb{R}^{S_n}$ and $f_{\bm{\psi}_{\bm{M}, \bm{\theta}} \ [-F,F]}^{\operatorname{DNN}} \in  \mathcal{F}^{\operatorname{DNN}}(L_n, \bm{s}_n, F)$ such that 
$f_{\bm{M}, \bm{\theta} \ [-F,F]}^{\operatorname{mDNN}} 
= f_{\bm{\psi}_{\bm{M}, \bm{\theta}} \ [-F,F]}^{\operatorname{DNN}}$ holds.
Since
\begin{align*}
\Big\| (\phi \circ f_{\bm{\psi}_1 \ [-F,F]}^{\operatorname{DNN}} - \phi \circ f_0)^2 
- (\phi \circ f_{\bm{\psi}_2 \ [-F,F]}^{\operatorname{DNN}} - \phi \circ f_0)^2 \Big\|_{n, 1}  \\
\leq 4 \left\|\phi \circ f_{\bm{\psi}_1 \ [-F,F]}^{\operatorname{DNN}}  - \phi \circ f_{\bm{\psi}_2 \ [-F,F]}^{\operatorname{DNN}}\right\|_{n, 1} \\
\leq 4 \left\|  f_{\bm{\psi}_1 \ [-F,F]}^{\operatorname{DNN}} - f_{\bm{\psi}_2 \ [-F,F]}^{\operatorname{DNN}}\right\|_{n, 1}
\end{align*}
holds for $\bm{\psi}_1, \bm{\psi}_2 \in \mathbb{R}^{S_n}$, there exists $c_4>0$ such that
\begin{align*}
\mathcal{N} \left( u, \mathcal{G}, ||\cdot||_{n, 1} \right)
\leq& \mathcal{N}\left( \frac{u}{4}, \mathcal{F}^{\operatorname{DNN}}(L_n, \bm{s}_n, F), ||\cdot||_{n, 1} \right) \\
\leq& \mathcal{M}\left( \frac{u}{4}, \mathcal{F}^{\operatorname{DNN}}(L_n, \bm{s}_n, F), ||\cdot||_{n, 1} \right) \\
\leq& 3\left(\frac{16 eF}{u} \log \frac{24 eF}{u}\right)^{V_{\mathcal{F}^{\operatorname{DNN}}(L_n, \bm{s}_n, F)}^{+}}\\
\lesssim& n^{c_4 S_n L_n \log S_n}
\end{align*}
for $u \geq n^{-1}$ by Theorem 9.4 in \citet{gyorfi2002distribution} and Theorem 6 of \citet{harvey2017nearly}.
Hence for all $t \geq \frac{\varepsilon_n^2}{8}$,
\begin{align*}
\int_{\frac{\kappa(1-\kappa)t}{16 \max \left\{K_{1}, 2 K_{2}\right\}}}^{\sqrt{t}}  
\sqrt{\log \mathcal{N} \left( u, \mathcal{G}, ||\cdot||_{n, 1} \right)} d u 
\lesssim& \sqrt{t} \left( n^{\frac{d}{2\beta + d}} (\log n)^{4+3\tau} \right)^{\frac{1}{2}}\\
=& o\left( \frac{\sqrt{n} t/4}{96 \sqrt{2} \max \left\{K_1, 2K_2 \right\}} \right)
\end{align*}
holds.
To sum up, we conclude that
\begin{align}
\mathbf{P}\left\{\sup _{f \in \mathcal{F}^{\operatorname{DNN}}(L_n, \bm{p}_n^{\prime})} \frac{\left| ||f-f_0||_{2, \mathrm{P}_{X}}^2  - ||f-f_0||_{2, n}^2 \right|}{\varepsilon_{n}^2+||f-f_0||_{2, \mathrm{P}_{X}}^2}>\frac{1}{2}\right\}
\leq 60 \exp \left(-\frac{n \varepsilon_{n} / 8}{128 \cdot 2304}\right).
\label{last3}
\end{align}
holds for all but finite many $n$ by Lemma \ref{gyorfi2002distribution}. 
Hence by (\ref{empirical result_2_2_cl}) and (\ref{last3}), the proof of (\ref{main result_2_1_cl}) is done. $\quad \square$

\subsection{Proof of Theorem \ref{cor1}}\label{appA_5}
Theorem holds by (\ref{main result_2_2}) for nonparametric regression (\ref{reg})
and by (\ref{main result_2_2_cl}) for binary classification (\ref{cla}). 
$\quad \square$

\newpage
\section{Theoretical results for hierarchical composition functions} \label{App_comp}
In Section \ref{sec4}, we only describe the results of Holder continuous functions for the sake of simplicity.
However, theoretical results for mBNN can be easily extended from the Holder continuous functions to the hierarchical compositional structured function considered in \citet{kohler2021rate}.
First, we define the function class of hierarchical composition functions.
\begin{definition}[hierarchical composition function]
Let $f : \mathbb{R}^d \to \mathbb{R}$ and $\mathcal{P}$ be a subset of $(0,\infty) \times \mathbb{N}$.
\begin{itemize}
\item[a)] We say that $f$ satisfies a hierarchical composition model of level 0 with order and
smoothness constraint $\mathcal{P}$, if there exists a $K \in [d]$ such that
$$f(\bm{x})=x^{(K)} \quad \text { for all } \bm{x}=\left(x^{(1)}, \ldots, x^{(d)}\right)^{\top} \in \mathbb{R}^d.$$
\item[b)] We say that $f$ satisfies a hierarchical composition model of level $q + 1$ with order and smoothness constraint $\mathcal{P}$, if there exist $(\beta, K) \in \mathcal{P}$, $g : \mathbb{R}^K \to \mathbb{R}$ and $f_1, \ldots, f_K: \mathbb{R}^d \to \mathbb{R}$ such that $g \in \mathcal{H}_K^\beta$, $f_1, \ldots, f_K$ satisfy a hierarchical
composition model of level $q$ with order and smoothness constraint $\mathcal{P}$ and
$$f(\bm{x})=g\left(f_1(\bm{x}), \ldots, f_K(\bm{x})\right) \quad \text { for all } \bm{x} \in \mathbb{R}^d.$$
\item[c)] For $q \in \mathbb{N}$ and $\mathcal{P} \subset (0,\infty) \times \mathbb{N}$, consider the hierarchical composition function of level $q$ with constraint $\mathcal{P}$, where for each function $g$ in the definition can be of different smoothness $p_g=q_g+s_g$ ($q_g \in \mathbb{N}_0$ and $s_g \in(0,1]$) and of different input dimension $K_g$, where $\left(p_g, K_g\right) \in \mathcal{P}$. 
Assume the maximal input dimension and the maximal smoothness
of $g$ are bounded.
Assume that each $g$ is Lipschitz continuous and all partial derivatives of order less than or equal to $q_g$ are bounded. 
We define the set of functions that satisfy these conditions as $\mathcal{H}(q,\mathcal{P})$.
\end{itemize}
\end{definition}

We describe a lemma that approximates hierarchical composition functions as DNN functions.
\begin{lemma}[Theorem 3 of \citet{kohler2021rate}]\label{kohler2021rate_comp}	
	There exists $C_L>0$ and $C_p>0$ only depending on $d$ such that for every $f_0 \in \mathcal{H}(q,\mathcal{P})$ with $\left\|f_0\right\|_{\infty} \leq F$, 
	there exist $f_{\hat{\bm{\theta}} \ [-F, F]}^{\operatorname{DNN}} \in \mathcal{F}^{\operatorname{DNN}}(L_n, \bm{v}_n, F)$ with
	\begin{align*}
	L_n :=& \left\lceil C_L \log n \right\rceil, \\
	v_n :=& \left\lceil C_p (\log n)^{-1/2} \max_{(\beta,K)\in \mathcal{P}} n^\frac{K}{2(2\beta + K)}  \right\rceil, \\
	\bm{v}_n :=& (d, v_n, ... , v_n, 1)^{\top} \in \mathbb{N}^{L_n + 2},
	\end{align*}
	such that
	\begin{align*}
	\left\|f_{\hat{\bm{\theta}} \ [-F, F]}^{\operatorname{DNN}}-f_{0}\right\|_{\infty} \lesssim \max_{(\beta,K)\in \mathcal{P}} n^{-\frac{\beta}{2\beta + K}} \log n 
	\end{align*}
	and
	\begin{align} 
    |\hat{\bm{\theta}}|_{\infty} \leq n \label{upper_param_comp}
	\end{align}
	hold.
\end{lemma}

In Section \ref{sec4}, we assume the true regression function (or the logit of the true conditional probability) $f_0$ belongs to the $\beta$-Holder class.
If $d$ is relatively large compared to $\beta$, concentration rate $n^{-\frac{\beta}{(2\beta + d)}}$ can be extremely slow.
However, mBNN can avoid curse of dimensionality for hierarchical composition function, which is demonstrated in the following theorem.
\begin{theorem}[Theoretical results for hierarchical composition functions]\label{theorem_comp}
    For $f_0 \in \mathcal{H}(q,\mathcal{P})$, Theorem \ref{theorem1} and \ref{theorem2} hold  with the concentration rate 
    $$\varepsilon_{n}= \max_{(\beta,K)\in \mathcal{P}} n^{-\beta /(2 \beta+K)} \log ^{\gamma}(n)$$ 
    for $\gamma > \frac{5}{2}$.
    Furthermore, Theorem \ref{cor1} holds with
    $$
    \mathfrak{s}_n := \left\lceil C_p (\log n)^{1/2}  \max_{(\beta,K)\in \mathcal{P}}n^\frac{K}{2(2\beta + K)}  \right\rceil
    $$   
\end{theorem}
\begin{proof}
    We can follow every step in Appendix \ref{appA_3}, \ref{appA_4} and \ref{appA_5} by simply changing the definition of $\varepsilon_{n}$ to
    $$\varepsilon_{n} = \max_{(\beta,K)\in \mathcal{P}} n^{-\beta /(2 \beta+K)} (\log n)^{\gamma}   
= \max_{(\beta,K)\in \mathcal{P}} n^{-\beta /(2 \beta+K)} (\log n)^{\tau + \frac{5}{2}}$$
and definition of $s_{n}$ to
\begin{align*}
s_n := \left\lceil C_p \left(\max_{(\beta,K)\in \mathcal{P}} n^\frac{K}{2\beta + K} (\log n)^{3\tau} \right)^{1/2} \right\rceil.
\end{align*}
\end{proof}

\newpage

\section{More detail for the mBNN}\label{appB}

\subsection{MCMC algorithm} \label{appB_1}
First we keep $\bm{M}^{(t)}$ fixed and update $\bm{\theta}^{(t)}$ (and ${\sigma^2}^{(t)}$) using existing MCMC algorithm, then we update $\bm{M}^{(t)}$ using MH algorithm.
In practical, we can update $\bm{M}^{(t)}$ using only the data in single mini-batch for large scale dataset.
Algorithm \ref{algMC} is a brief summary of our algorithm.
\begin{algorithm}[h]
	\caption{Proposed MCMC algorithm} \label{algMC}
	\textbf{INPUT:} $T_{\operatorname{total}}, T_{\operatorname{MH}}, n_{\operatorname{MH}} \in \mathbb{N}$
	\begin{algorithmic}[1]
	    \FOR{$t=1$ to $T_{\operatorname{total}}$}
		\STATE Update $\bm{\theta}^{(t)}$ (and ${\sigma^2}^{(t)}$) using existing MCMC algorithm (e.g. HMC, SGLD).
		\IF{$t \ \% \ T_{\operatorname{MH}}= 0$}
		\STATE Update $\bm{M}^{(t)}$ using Algorithm \ref{algMH}, $n_{\operatorname{MH}}$ times.
		\ENDIF
		\ENDFOR
	\end{algorithmic}
\end{algorithm}

\subsection{Masked Bayesian CNN} \label{appB_2}
First, we introduce how to apply masking variables for masked Bayesian CNN.
Most of CNN architectures consist of a mixture of sequences of convolution layer and RELU activation function.
For a given CNN, the corresponding masked CNN is constructed by simply adding masking parameters to the CNN model.
For the $l$-th convolution layer, the masked CNN screens the channels using binary masking vector whose dimension is equal to the number of channels in $l$-th layer.
Figure \ref{figmcnn} is a illustration of the masked convolution layer.

\begin{figure*}[h]
\centering
\subfigure[Original convolution layer]{\includegraphics[width=0.45\linewidth]{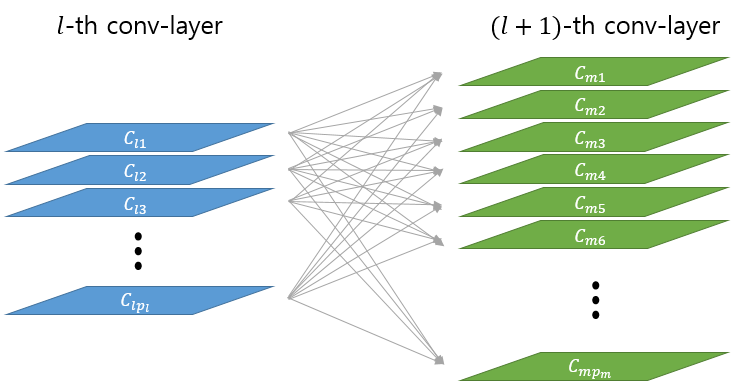}}
\hfill
\subfigure[Convolution layer with masking vector]{\includegraphics[width=0.45\linewidth]{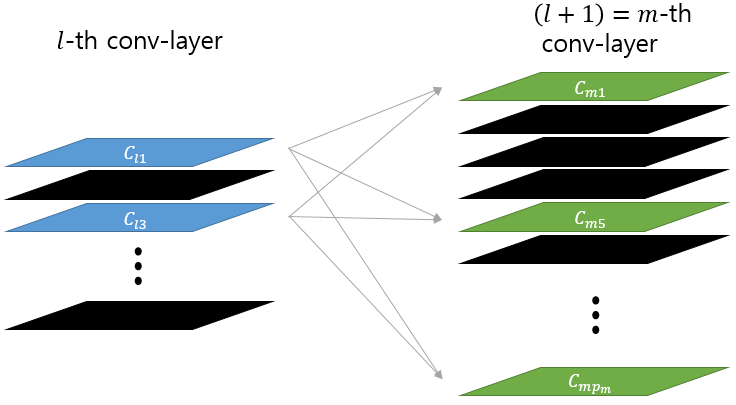}}
\caption{\textbf{Convolution layer and masked convolution layer.}  Masking vector $\bm{m}^{(l)}$ screens channels of the $l$-th layer. When $\left(\bm{m}^{(l)}\right)_j = 0,$ 
the $j$-th channel of the $l$-th layer becomes inactive.
Blue and green rectangles represent input and output channels, and black rectangles represent masked channels. 
} \label{figmcnn}
\end{figure*}

As we did in Section \ref{sec_3.2}, we can simply implement this by converting RELU activation functions into masked-RELU activation functions. 
For instance, since Resnet18 structure has 17 RELU activation functions, we use 17 masking vectors whose dimensions are equal to the number of channels of the corresponding hidden layer.

Then, the prior $\Pi_n$ given as (\ref{prior_s}), (\ref{prior_m}) and (\ref{prior_theta}) is adopted on parameters $\bm{M}$ and $\bm{\theta}$, where $\bm{M}$ is the concatenate of the all masking vectors and $\bm{\theta}$ is the concatenation of the filters, weight matrices and bias vectors in the CNN model.

\newpage
\section{Detailed settings for the experiments} \label{App_setting}
\renewcommand{\theequation}{C.\arabic{equation}}
\subsection{Noisy polynomial regression}
\citet{hernandez2015probabilistic} considers a noisy polynomial regression problem, where
the input $X$ and output $Y$ are sampled from
\begin{align}
\begin{split} \label{poly_setting}
X \sim& \operatorname{Uniform}(-4,4), \\
Y =& X^3 + \epsilon, \qquad \qquad \epsilon \sim N(0,9). \\
\end{split}
\end{align}
Following their setting, We first generate 20 training examples from (\ref{poly_setting}), and compare the predictive distribution of the mBNN with that of BNN. 
For the both methods, the two hidden layer MLP with the layer sizes (1000,1000) and Cauchy prior with scale 0.3 are used.
The prediction distributions are approximated by 1000 MCMC samples which are obtained by HMC with the step size $10^{-2}$ and NUTs \cite{hoffman2014no}, 300 burn-in samples and thinning interval 10.
For the mBNN, we use $N_{\max} = 3$, $\lambda=0.1$, $T_{MH}=1$ and $n_{MH}=2$.

After we obtain the predictive distributions,  we generate 1000 test examples from (\ref{poly_setting}).
For each method and test examples, we obtain the 95\% predictive interval of the predictive distribution.
That is, for the $i$-th test example $x^{(i)}$, we sample $y^{(i)}_1, \dots, y^{(i)}_N$ from the predictive distribution (which is expressed as a Gaussian mixture distribution where the mixing distribution is the empirical distribution on the MCMC samples), and then obtain the 95\% predictive interval by the 2.5\% and 97.5\% quantiles among $y^{(i)}_1, \dots, y^{(i)}_N.$

\subsection{UCI dataset}
For BNN and the mBNN, two hidden layer MLP with the layer sizes (1000,1000) and Cauchy prior with scale 1.0 are used.
The prediction distributions are approximated by 20 MCMC samples which are obtained
by HMC with step size $10^{-2}$ and NUTs, 2 burn-in samples and thinning interval 200.
For the mBNN, we use $N_{\max} = 3$, $\lambda=0.1$, $T_{MH}=1$ and $n_{MH}=10$.

For the node-sparse VI (NS-VI), we adopt Bayesian compression using group normal-Jeffreys prior \cite{louizos2017bayesian}. 
Note that most of existing node-sparse VI have the same spirit with \citet{louizos2017bayesian} in the sense of using hierarchical scale-mixture priors to prune nodes.
Among those algorithms, this method is chosen in consideration of versatility and reproducibility.
For UCI datasets, we use Adam \cite{kingma2014adam} optimizer with learning rate 0.01, 500 epoch and batch size 100.

\subsection{Bayesian structural time series model} \label{appendix_bsts}
The state space model with a general regression function $f_{\bm{\psi}}(\cdot)$, which is  parameterized by $\bm{\psi}$, can be expressed as 
\begin{align*}
y_t &= \mu_t + f_{\bm{\psi}}(\bm{x}_t) + \epsilon_t , && \epsilon_t \sim \mathcal{N}(0,\sigma_{\epsilon}^2)\\
\mu_{t+1} &= r \mu_{t} + \eta_t  , && \eta_t \sim \mathcal{N}(0,\sigma_{\eta}^2)\\
\mu_0 &\sim N(a_0, \sigma_0^2)
\end{align*}
where $0<r<1$ and $\bm{x}_t \in \mathbb{R}^d$, $y_t \in \mathbb{R}$ and $\mu_t \in \mathbb{R}$ denote observed input vector, output variable and unobserved local trend at time $t$, respectively.
Based on the state space model, BSTS model aims to forecast the current value of $y$ using 
the current values of other variables. 
This problem is called “nowcasting”\cite{banbura2010nowcasting}, which is widely used for economics \cite{giannone2008nowcasting}, inferring causal relationships \cite{brodersen2015inferring}, anomaly detection \cite{feng2021time}, to name just a few.

The dataset consists of daily search volumes of keywords in year 2021 associated with a pre-specified product (eg. shampoo).
We aim to predict the level of interest of the pre-specified product based on the past and current search volumes of related keywords. To be more specific, $y_t$ is the daily search volumes of the pre-specified product
and $\bm{x}_t$ is the vector of the daily search volumes of the related keywords.

For regression function $f_{\bm{\psi}}(\bm{x})$, we consider the linear, BNN and the mBNN models.
For the linear model, we adopt i.i.d standard Gaussian prior on  $\bm{\psi}$,
while for BNN and the mBNN we adopt the priors considered in the paper. 
We use inverse-gamma prior on $\sigma_{\epsilon}^2$ and standard Gaussian prior on $\mu_0$.
We use $r=0.95$ and $\sigma_{\eta}^2 = 0.1$.
Using the past dataset $\{\bm{x}_t , y_t\}_{t=1}^{T}$ where $T=240$,
we generate MCMC samples of $\bm{\psi}$, $\sigma_{\epsilon}^2$ and $\{\mu_t\}_{t=1}^T$ from the posterior distribution using an MCMC algorithm.
To be more specific, we first sample $\bm{\psi}$ from  its conditional posterior distribution.
For the  linear model, $\bm{\psi}$ can be  sampled easily since its conditional posterior distribution is also Gaussian.
For BNN and the mBNN, we use the MCMC algorithm developed in this paper.
Once $\bm{\psi}$ is sampled,
we sample $\sigma_{\epsilon}^2$ and sample $\mu_t$ using the Kalman filter \cite{welch1995introduction, durbin2002simple}. 

For every method, the prediction distributions are approximated by 20 MCMC samples with 10 burn-in samples and thinning interval 10. 
We use standard Gaussian prior for the linear model, and 
For BNN and the mBNN, two hidden layer MLP with the layer sizes (100,100) and Cauchy prior with scale 1.0 are used.
For the mBNN, we use $N_{\max} = 3$, $\lambda=0.1$, $T_{MH}=1$ and $n_{MH}=2$.

\subsection{Image dataset} \label{app_image}
For NS-VI, starting from a pretrained network,  
we use the Adam optimizer with the learning rate 0.0001, 200 epoch and batch size 100 for optimization.

For node-sparse Deep ensemble (NS-Ens), we first train each member of the ensemble to be node-wise sparse DNNs using LeGR \cite{chin2020towards},
and then combine them  as \citet{lakshminarayanan2017simple} does.
We use the default hyper-parameters provided in \cite{chin2020towards}.

As MC-dropout \cite{gal2016dropout} uses the weights obtained by Dropout \cite{srivastava2014dropout} multiplied by a random masking vectors, 
we use the node-wise sparse weights obtained by Targeted Dropout \cite{gomez2019learning} multiplied by a random masking vectors for node-sparse MC-Dropout (NS-MC).
Specifically, let $\gamma \in [0,1]$ and $\alpha \in [0,1]$ be the pre-specified proportions for pruning and drop probabilities, respectively.
Before each gradient step, we select nodes having the $\gamma \times 100 \%$ lowest magnitude ($L^2$-norm of the connected weights) on each layer, and then randomly mask the selected nodes with probability $\alpha$.
This implies that the expected ratio of nodes to survive during each step is $(1-\alpha \gamma)$.
For inference, we multiply a random masking vector to the nodes.
Through careful tuning, we choose $(\alpha,\gamma) = (0.9, 0.4)$ for CIFAR10 and $(\alpha,\gamma) = (0.9, 0.3)$ for CIFAR100. 
For optimization, we use the Adam optimizer with the learning rate 0.001, 200 epoch and batch size 100.

For the mBNN, we use SGLD \cite{welling2011bayesian} with the batch-size 100, step size $10^{-3}$ with the Cosine scheduler and temperature $1/\sqrt{n}$, which are commonly used in SG-MCMC literature \cite{zhang2019cyclical, wenzel2020good}. 
MCMC samples are obtained after 5 burn-in samples and thinning interval 20.
We use $N_{\max} = 3$, $\lambda=0.05$, $T_{MH}=10$ (batch) and $n_{MH}=1$.

\newpage
\section{Additional experimental results} \label{App_additional}

\paragraph{Additional results on various $\lambda$} 
In Section \ref{sec6_2}, we conduct an experiment to demonstrate the necessity
of masking variables on real dataset, where
we report the result of the mBNN with $\lambda=0.1$.
Asymptotically the mBNN automatically finds 
an appropriate network architecture for given data and complexity of the true regression model
as long as $\lambda$ remains bounded, but
finite sample performance varies according to the choice of $\lambda.$

\begin{table}[h!] 
  \caption{\textbf{Additional results on various $\lambda$.} Performance of BNN, NS-VI and the mBNN (with various $\lambda$) on UCI datasets.}
  \begin{adjustbox}{center,max width=\linewidth}
  \setlength{\aboverulesep}{0.2pt}
  \setlength{\belowrulesep}{0.2pt}
  \scalebox{0.9}{
    \begin{tabular}{c|c|c|ccccc}
      \hline
       \textbf{Dataset} & \textbf{Method} & $\lambda$ & \textbf{R in 95\% C.I.} &  \textbf{RMSE} & \textbf{NLL} & \textbf{CRPS} & \textbf{\# of activated nodes} \\
    \hline
    \multirow{5}{*}{Boston} 
    &  BNN  & $\cdot$ & 0.912(0.006) & 3.411(0.145) & 2.726(0.087) & 1.738(0.052) & (1000,1000)  \\
    &  NS-VI  & $\cdot$ & 0.746(0.010) &  3.079(0.179) & 4.242(0.568) & 1.661(0.069) & (13, 21) \\
    & mBNN & $0.075$ & 0.924(0.142) & 2.961(0.142) & 2.547(0.079) & 1.535(0.057) & (102, 44) \\
      &  mBNN & $0.1$ & 0.933(0.007) & 2.902(0.143) & 2.472(0.085) &  1.462(0.055)  & (29, 12)  \\
      & mBNN & $0.125$ & 0.939(0.061) & 2.790(0.170) & 2.381(0.063) & 1.431(0.061) & (11, 5) \\
    \hline
      \multirow{5}{*}{Concrete}   
      & BNN  & $\cdot$ & 0.908(0.008) & 5.080(0.178) & 3.091(0.055) & 2.658(0.072) & (1000,1000) \\
    &  NS-VI  & $\cdot$ & 0.568(0.011) & 5.046(0.149) & 9.713(0.686) & 2.965(0.088) & (30, 19)  \\
    & mBNN & $0.075$ & 0.916(0.127) & 4.908(0.127) & 3.023(0.034) & 2.626(0.062) & (108,54) \\
        & mBNN & $0.1$ & 0.912(0.009) & 4.913(0.180) & 3.027(0.046) & 2.628(0.087)  & (35, 17) \\
    & mBNN  & $0.125$ & 0.924(0.006) & 5.002(0.139) & 3.035(0.035) & 2.683(0.068) & (16, 7) \\
    \hline
     \multirow{5}{*}{Energy}
     & BNN	 & $\cdot$ & 0.945(0.005) & 0.591(0.017) & 0.902(0.025) & 0.322(0.007) & (1000,1000) \\
    &  NS-VI  & $\cdot$ & 0.913(0.006) & 1.322(0.117) & 1.792(0.121) & 0.720(0.055) & (7, 10)  \\
    & mBNN & $0.075$ & 0.943(0.005) & 0.473(0.012) & 0.676(0.030) & 0.254(0.005) & (211, 210) \\
      & mBNN & $0.1$ & 0.945(0.004) & 0.474(0.015) & 0.670(0.034) & 0.256(0.006)  & (35, 17)\\
      & mBNN  & $0.125$ & 0.955(0.004) & 0.510(0.026) & 0.736(0.047) & 0.278(0.013) & (19, 13)\\
     \hline
     \multirow{5}{*}{Yacht}
     & BNN	& $\cdot$ & 0.977(0.006) & 0.675(0.048) & 1.011(0.056) & 0.332(0.013) & (1000,1000) \\
    &  NS-VI  & $\cdot$ & 0.932(0.011) & 1.842(0.096) & 1.915(0.063) & 0.969(0.042) & (102, 54) \\
    & mBNN  & $0.075$ & 0.976(0.006) & 0.604(0.054) & 0.920(0.084) & 0.295(0.015) & (317, 193) \\
      & mBNN  & $0.1$ & 0.953(0.008) & 0.664(0.044) & 0.932(0.062) & 0.318(0.015)  & (120, 76)\\
      & mBNN  & $0.125$ & 0.954(0.011) & 0.621(0.054) & 0.966(0.102) & 0.301(0.023) & (42,26) \\
    \hline
    \end{tabular}
    }
  \end{adjustbox}
 \label{table_UCI_appendix}
\end{table}

In Table \ref{table_UCI_appendix}, we report the means and standard errors of the performance measures of the mBNN
with various values of $\lambda.$ The results indicate that the mBNN performs stably with respect to the choice of
$\lambda$ even though the level of sparsity is proportional to $\lambda.$ In practice,
$\lambda$ can be selected based on either validation data or putting a prior on $\lambda.$

\newpage

\paragraph{Compatibility of proposed MCMC algorithm with mini-batching}
The bias of MH with mini-batch has been discussed, and recently there have been many works on this topic. Specifically, TunaMH \cite{zhang2020asymptotically} and MHBT \cite{wu2022mini} are representative examples, and both algorithms ensure the convergence of mini-batch MH. Both algorithms propose unbiased mini-batch MH through appropriate modifications of batch-sampling method and acceptance rate. That is, their methods can be applied to our MH algorithm without much modification. We conduct experiments with mBNN utilizing MHBT on CIFAR10 and CIFAR100 to obtain the following results in Table \ref{MBNN_MHBT}.
The results are similar to the original MH algorithm (without caring biases due to mini-batches).
\begin{table}[h!] \label{MBNN_MHBT}
  \caption{\textbf{mBNN utilizing MHBT} Performance of mBNN utilizing MHBT on image datasets.}
  \centering
\begin{tabular}{c|c|c|c|c|c}
\hline
Dataset  & ACC   & NLL   & ECE   & FLOPs & Capacity \\
\hline
CIFAR10  & 0.933 & 0.206 & 0.006 & 12.92\% & 3.82\%   \\
\hline
CIFAR100 & 0.740 & 0.982 & 0.002 & 21.52\% & 14.27\% \\
\hline
\end{tabular}
\end{table}

\paragraph{Additional results for Section \ref{sec6_5}}

In Figure \ref{proposal_compare}, we report the results only on Yacht, CIFAR10 and CIFAR100 datasets to illustrate the efficiency of our proposal distribution in Algorithm \ref{algMH}.
Here, we report the results for the other datasets which are presented in Figure \ref{proposal_compare_Appendix}. 
The patterns are similar in that our proposal distribution works most efficiently.

\begin{figure}[h!]
\centering
\subfigure[Noisy polynomial]{\includegraphics[width=0.24\linewidth]{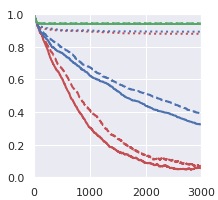}}
\hfill
\subfigure[Boston]{\includegraphics[width=0.24\linewidth]{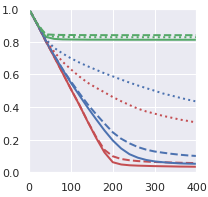}}
\hfill
\subfigure[Concrete]{\includegraphics[width=0.24\linewidth]{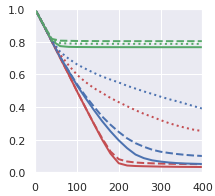}}
\hfill
\subfigure[Energy]{\includegraphics[width=0.24\linewidth]{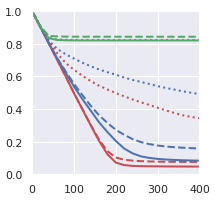}}
\hfill\\
\hfill
\subfigure[Yacht]{\includegraphics[width=0.24\linewidth]{image/manuscript/Yacht_9method.PNG}}
\hfill
\subfigure[CIFAR10]{\includegraphics[width=0.24\linewidth]{image/manuscript/CIFAR10_9method.PNG}}
\hfill
\subfigure[CIFAR100]{\includegraphics[width=0.24\linewidth]{image/manuscript/CIFAR100_9method.PNG}}
\hfill
\caption{\textbf{Efficiency of the proposal.} 
The ratios of activated nodes (y-axis) relative to the largest DNN
are presented as the MCMC algorithms iterate in the burn-in period (x-axis).
The solid, dashed and dotted lines correspond to (\ref{pro_uni}), (\ref{pro_reci}) and (\ref{pro_line}) for {\it birth}, respectively while blue, red and green lines correspond to (\ref{pro_uni}), (\ref{pro_reci}) and (\ref{pro_line}) for {\it death}, respectively. 
} \label{proposal_compare_Appendix}
\end{figure}

\end{document}